%% file: main_neurips.tex
\documentclass{article}

\PassOptionsToPackage{semicolon,round}{natbib}

\usepackage[final]{neurips_2023}

\usepackage{titletoc}
\usepackage[page, header, toc, page]{appendix} % MAKE SURE THIS IS LOADED BEFORE hyperref PACKAGE!

\usepackage[utf8]{inputenc} % allow utf-8 input
\usepackage[T1]{fontenc}    % use 8-bit T1 fonts
\usepackage{hyperref}       % hyperlinks
\usepackage{url}            % simple URL typesetting
\usepackage{booktabs}       % professional-quality tables
\usepackage{amsfonts}       % blackboard math symbols
\usepackage{nicefrac}       % compact symbols for 1/2, etc.
\usepackage{microtype}      % microtypography
\usepackage{xcolor}         % colors

\usepackage{amsmath}
\usepackage{amssymb}
\usepackage{amsthm}

\usepackage{bm}
\usepackage{mathtools}
\usepackage{mdframed}
\usepackage{microtype}
\usepackage{thmtools} 
\usepackage{todonotes}
\usepackage{caption}
\usepackage{subcaption}
\usepackage{yfonts}

\input{command}

\input{_macros}

\usepackage{wrapfig}

\usepackage{tcolorbox}
\usepackage[normalem]{ulem}

\makeatletter
\renewcommand{\paragraph}{%
\@startsection{paragraph}{4}%
{\z@}{1.25ex \@plus 1ex \@minus .2ex}{-1em}%
{\normalfont\normalsize\bfseries}%
}
\newcommand*{\centerfloat}{%
\parindent \z@
\leftskip \z@ \@plus 1fil \@minus \textwidth
\rightskip\leftskip
\parfillskip \z@skip}
\makeatother

\setlist{itemsep=0pt,parsep=2pt,topsep=0pt,leftmargin=*}

\allowdisplaybreaks

\title{The Crucial Role of Normalization in \\ Sharpness-Aware Minimization}

\author{% 
Yan Dai~\thanks{The first two authors contribute equally. Work done while Yan Dai was visiting MIT.} \\
IIIS, Tsinghua University \\
\texttt{yan-dai20@mails.tsinghua.edu.cn}
\And
Kwangjun Ahn~\footnotemark[1] \\
EECS, MIT \\
\texttt{kjahn@mit.edu}
\And
Suvrit Sra \\
TU Munich / MIT\\
\texttt{suvrit@mit.edu}
}

\begin{document}

\maketitle

\setcounter{footnote}{0}

\begin{abstract}
Sharpness-Aware Minimization (SAM) is a recently proposed gradient-based optimizer (Foret et al., ICLR 2021) that greatly improves the prediction performance of deep neural networks. Consequently, there has been a surge of interest in explaining its empirical success. We focus, in particular, on understanding \emph{the role played by normalization}, a key component of the SAM updates.
We theoretically and empirically study the effect of normalization in SAM for both convex and non-convex functions, revealing two key roles played by normalization: i) it helps in stabilizing the algorithm; and ii) it enables the algorithm to drift along a continuum (manifold) of minima -- a property identified by recent theoretical works that is the key to better performance.
We further argue that these two properties of normalization make SAM robust against the choice of hyper-parameters, supporting the practicality of SAM. 
Our conclusions are backed by various experiments.
\end{abstract}

\section{Introduction}
\label{sec:intro}

We study the recently proposed gradient-based optimization algorithm \emph{Sharpness-Aware Minimization (SAM)}~\citep{foret2020sharpness} that has shown impressive performance in training deep neural networks to generalize well~\citep{foret2020sharpness,bahri2022sharpness,mi2022make,zhong2022improving}. 
SAM updates involve an ostensibly small but key modification to Gradient Descent (GD). Specifically, for a loss function $\loss$ and each iteration $t\ge 0$, instead of updating the parameter $w_t$ as $w_{t+1}=w_t-\eta \nabla \loss (w_t)$ (where $\eta$ is called the \emph{learning rate}),  SAM performs the following update:\footnote{In principle, the normalization in \autoref{eq:definition of SAM} may make SAM ill-defined.
However, \citet[Appendix B]{wen2022does} showed that except for countably many learning rates, SAM (with any $\rho$) is always well-defined for almost all initialization.
Hence, throughout the paper, we assume that the SAM iterates are always well-defined.\label{remark:ill-definedness}}  
\begin{equation}  
w_{t+1}=w_t-\eta \nabla \loss\left (w_t+\rho \frac{\nabla \loss(w_t)}{\lVert \nabla \loss(w_t)\rVert}\right )\,,\label{eq:definition of SAM}
\end{equation}
where $\rho$ is an additional hyper-parameter that we call the \emph{perturbation radius}.
\citet{foret2020sharpness} motivate SAM as an algorithm minimizing the robust loss $\max_{\lVert \epsilon\rVert\le \rho}\loss(w+\epsilon)$, which is roughly the loss at $w$ (i.e., $\loss(w)$) plus the ``sharpness'' of the loss landscape around $w$, hence its name.

The empirical success of SAM has driven a recent surge of interest in characterizing its dynamics and theoretical 
properties~\citep{bartlett2022dynamics,wen2022does,ahn2023escape}.
However, a major component of SAM remains unexplained in prior work: the role and impact of the normalization factor $\frac{1}{\lVert \nabla \loss(w_t)\rVert}$ used by SAM. In fact, quite a few recent works drop the normalization factor for simplicity when analyzing SAM~\citep{andriushchenko2022towards,behdin2023sharpness,agarwala2023sam,kim2023stability,compagnoni2023sde}. Instead of the SAM update~\eqref{eq:definition of SAM}, these works consider the following update that we call \emph{Un-normalized SAM (USAM)}:
\begin{equation} 
w_{t+1}=w_t-\eta \nabla \loss\left (w_t+\rho  \nabla \loss(w_t) \right)\,.\label{eq:definition of USAM}
\end{equation}
Apart from experimental justifications in \citep{andriushchenko2022towards}, the effect of this simplification has not yet been carefully investigated, although it is already widely adopted in the community.
Thus, is it really the case that such normalization can be omitted ``for simplification'' when theoretically analyzing SAM?
These observations raise our main question:
\begin{center}
\textbf{\emph{What is the role of the normalization factor $\frac{1}{\lVert \nabla \loss(w_t)\rVert}$ in the SAM  update~\eqref{eq:definition of SAM}?}}
\end{center}

\subsection{Motivating Experiments and Our Contributions}
We present our main findings through two motivating experiments.
For the setting, we choose the well-known over-parameterized matrix sensing problem \citep{li2018algorithmic}; see \autoref{app:exp_detail} for details. 

\begin{enumerate}
\item  {\bf Normalization helps with stability.}
We first pick a learning rate $\eta$ that allows GD to converge, and we gradually increase $\rho$ from $0.001$ to $0.1$. Considering the early stage of training shown in \autoref{fig:stability}. One finds that \emph{SAM has very similar behavior to GD}, whereas \textit{USAM  diverges even with a small $\rho$} -- it seems that normalization helps stabilize the algorithm. 
\begin{figure}[H]
\centering
\includegraphics[width=0.3\textwidth]{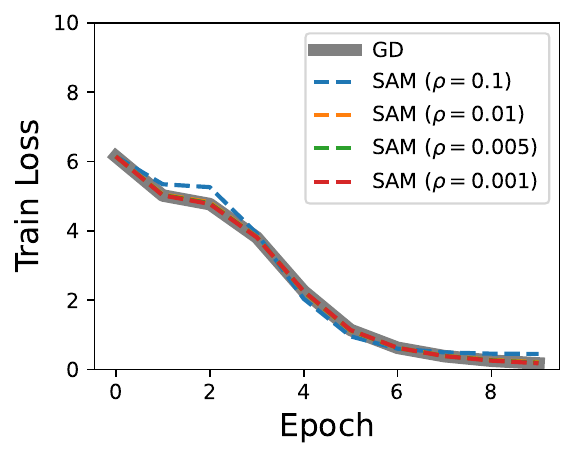}
\includegraphics[width=0.28\textwidth]{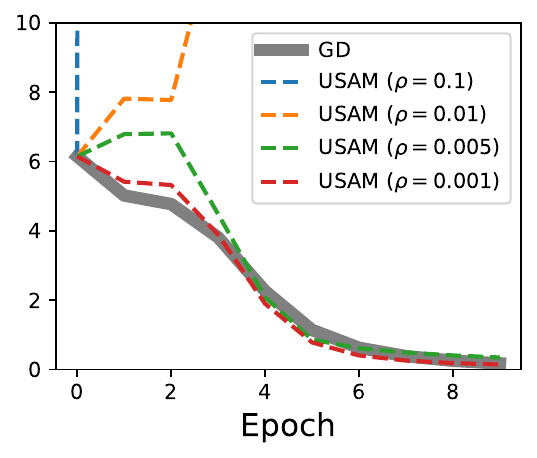}
\caption{{Role of normalization for stabilizing algorithms} \footnotesize{($\eta = 0.05$).} }
\label{fig:stability}
\end{figure} 

\item {\bf Normalization permits moving along minima.}
We reduce the step size by $10$ times and consider their performance of reducing test losses in the long run.  
One may regard \autoref{fig:solution} as the behavior of SAM, USAM, and GD when close to a ``manifold'' of minima (which exists since the problem is over-parametrized) as the training losses are close to zero.
The first plot compares SAM and USAM with the same $\rho=0.1$ (the largest $\rho$ for which USAM doesn't diverge): notice that USAM and GD both converge to a minimum and do not move further;
on the other hand, SAM keeps decreasing the test loss, showing its ability to drift along the manifold. 
We also vary $\rho$ and compare their behaviors (shown on the right): \textit{the ability of SAM to travel along the manifold of minimizers seems to be robust} to the size of $\rho$, while \textit{USAM easily gets stuck at a minimum}.

\begin{figure}[H]
\centering
\includegraphics[width=0.3\textwidth]{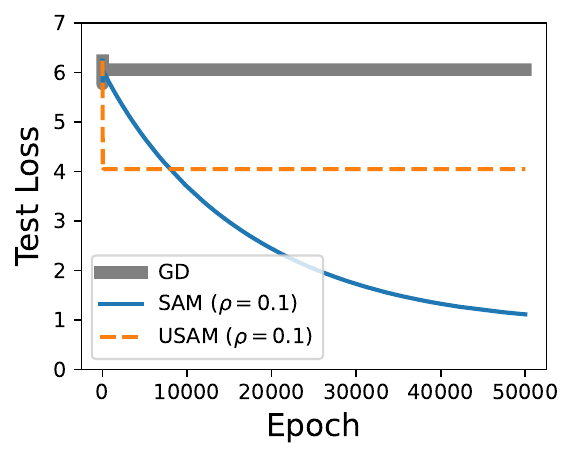}
\includegraphics[width=0.28\textwidth]{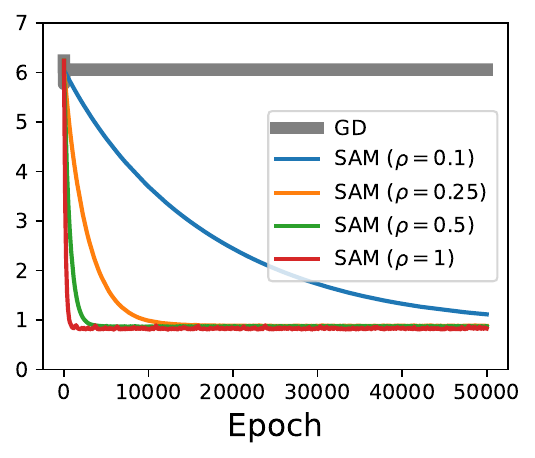}
\includegraphics[width=0.28\textwidth]{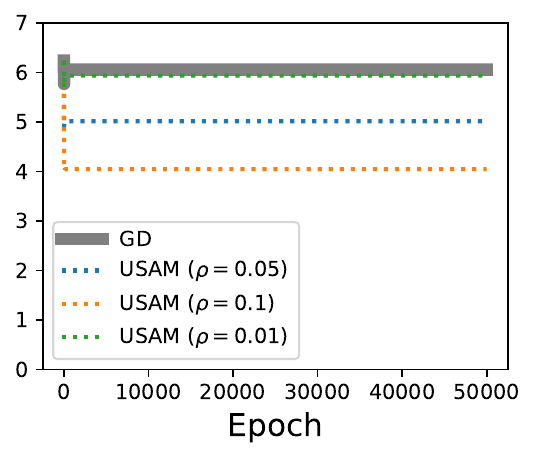}
\caption{{Role of normalization when close to a manifold of minimizers}\footnotesize{ ($\eta = 0.005$).} }
\label{fig:solution}
\end{figure}
\end{enumerate}

\textbf{Overview of our contributions.}
In this work, as motivated by \autoref{fig:stability} and \autoref{fig:solution}, we identify and theoretically explain the two roles of normalization in SAM.  
The paper is organized as follows.
\begin{enumerate}[itemsep=2pt]
\item In \autoref{sec:stability}, we study the role of normalization in the algorithm's stability and show that normalization helps stabilize.
In particular, we demonstrate that normalization ensures that GD's convergence implies SAM's non-divergence, whereas USAM can start diverging much earlier.

\item In \autoref{sec:solution}, we study the role of normalization near a manifold of minimizers and show that the normalization factor allows iterates to keep drifting along this manifold -- giving better performance in many cases.
Without normalization, the algorithm easily gets stuck and no longer progresses.

\item In \autoref{sec:sparse_coding}, to illustrate our main findings, we adopt the sparse coding example of \citet{ahn2022learning}.
Their result implies a dilemma in hyper-parameter tuning for GD: a small $\eta$ gives worse performance, but a large $\eta$ results in divergence. We show that this dilemma extends to USAM -- but not SAM.
In other words, SAM easily solves the problem where GD and USAM often fail.
\end{enumerate}
These findings also shed new light on why SAM is practical and successful, as we highlight below.

\textbf{Practical importance of our results.}
The main findings in this work explain and underscore several practical aspects of SAM that are mainly due to the normalization step.
One practical feature of SAM is the way the hyper-parameter $\rho$ is tuned: \citet{foret2020sharpness} suggest that $\rho$ can be tuned independently after tuning the parameters of base optimizers (including learning rate $\eta$, momentum $\beta$, and so on).
In particular, this feature makes SAM a perfect ``add-on'' to existing gradient-based optimizers.
Our findings precisely support this practical aspect of SAM. Our results suggest that \textbf{\emph{the stability of SAM is less sensitive to the choice of $\rho$, thanks to the normalization factor.}}

The same principle applies to the behavior of the algorithm near the minima: Recent theoretical works \citep{bartlett2022dynamics,wen2022does,ahn2023escape} have shown that the drift along the manifold of minimizers is a main feature that enables SAM to reduce the sharpness of the solution (which is believed to boost generalization ability in practice) -- our results indicate that \textbf{\emph{the ability of SAM to keep drifting along the manifold is independent of the choice of $\rho$, again owing to normalization.}}
Hence, our work suggests that the normalization factor plays an important role towards SAM's empirical success.

\subsection{Related Work}
\textbf{Sharpness-Aware Optimizers.}
Inspired by the empirical and theoretical observation that the generalization effect of a deep neural network is correlated with the ``sharpness'' of the loss landscape (see \citep{keskar2017large,jastrzkebski2017three,jiang2019fantastic} for empirical observations and \citep{dziugaite2017computing,neyshabur2017exploring} for theoretical justifications),
several recent papers \citep{foret2020sharpness,zheng2021regularizing,wu2020adversarial} propose optimizers that penalize the sharpness for the sake of better generalization.
Subsequent efforts were made on making such optimizers scale-invariant \citep{kwon2021asam}, more efficient \citep{liu2022towards,du2022efficient}, and generalize better \citep{zhuang2022surrogate}.
This paper focuses on the vanilla version proposed by \citet{foret2020sharpness}.

\textbf{Theoretical Advances on SAM.}
Despite the success of SAM in practice, theoretical understanding of SAM was absent until two recent works: 
\citet{bartlett2022dynamics} analyze SAM on locally quadratic losses and identify a component reducing the sharpness $\lambda_{\max}(\nabla^2 \loss(w_t))$, while \citet{wen2022does} characterize SAM near the manifold $\Gamma$ of minimizers and show that SAM follows a Riemannian gradient flow reducing $\lambda_{\max}(\nabla^2 \loss(w))$ when i) initialized near $\Gamma$, and ii) $\eta$ is ``small enough''. Note that while the results of \citet{wen2022does} apply to more general loss functions, our result in \Cref{thm:SAM ell(xy) formal} applies to i) any initialization far from the origin, and ii) any $\eta=o(1)$ and $\rho=O(1)$. 
A recent work by \citet{ahn2023escape} formulates the notion of $\eps$-approximate flat minima and analyzed the iteration complexity of practical algorithms like SAM to find such approximate flat minima.
A concurrent work by \citet{si2023practical} also analyzes the original version of SAM with the normalization in \eqref{eq:definition of SAM}, and makes a case that practical SAM does not converge all the way
to optima.

\textbf{Unnormalized SAM (USAM).}
USAM was first proposed by \citet{andriushchenko2022towards} who observed a similar performance between USAM and SAM when training ResNet over CIFAR-10.
This simplification is further accepted by \citet{behdin2023sharpness} who study the regularization effect of USAM over a linear regression model, by
\citet{agarwala2023sam} who study the initial and final dynamics of USAM over a quadratic regression model, and by
\citet{kim2023stability} who study the convergence instability of USAM near saddle points. To our knowledge, \citep{compagnoni2023sde} is the only work comparing SAM and USAM dynamics. 
More preciously, they consider the continuous-time behavior of SGD, SAM, and USAM and find different behaviors of SAM and USAM: USAM attracts local minima while SAM aims at global ones. Still, we remark that as they are considering continuous-time variants of algorithms while we consider discrete (original) versions, our results directly apply to the SAM deployed in practice and the USAM studied in theory.

\textbf{Edge-of-Stability.}
In the optimization theory literature, Gradient Descent (GD) was only shown to find minima if the learning rate $\eta$ is smaller than an ``Edge-of-Stability'' threshold, which is related to the sharpness of the nearest minimum.
However, people recently observe that when training neural networks, GD with a $\eta$ much larger than that threshold often finds good minima as well (see \citep{Cohen2021} and references therein).
Aside from convergence, GD with large $\eta$ is also shown to find \textit{flatter} minima \citep{arora2022understanding,ahn2022understanding,wang2022analyzing,damian2022self}. 

\section{Role of Normalization for Stability}
\label{sec:stability} 

In this section, we discuss the role of normalization in the stability of the algorithm.
We begin by recalling a well-known fact about the stability of GD: for a convex  quadratic cost with the largest eigenvalue of Hessian being $\beta$ (i.e., $\beta$-smooth), GD  converges to a minimum iff $\eta <\nicefrac{2}{\beta}$. 
Given this folklore fact, we ask: how do the ascent steps in SAM \eqref{eq:definition of SAM} and USAM \eqref{eq:definition of USAM} affect their stability?

\subsection{Strongly Convex and Smooth Losses}
\label{sec:quadratics}

Consider an $\alpha$-strongly-convex and $\beta$-smooth loss function $\loss$ where GD is guaranteed to converge once $\eta < \nicefrac 2\beta$.
We characterize the stability of SAM and USAM in the following result.

\begin{theorem}[Strongly Convex and Smooth Losses]\label{thm:Strongly Convex and Smooth}
For any $\alpha$-strongly-convex and $\beta$-smooth loss function $\loss$, for any learning rate $\eta < \nicefrac 2\beta$ and perturbation radius $\rho\geq 0$, the following holds:
\begin{enumerate}
\setlength{\itemsep}{0pt}
\item \textbf{SAM.} The iterate $w_t$  converges to a local neighborhood around the minimizer $w^\star$. Formally,
\begin{align}\label{eq:convergence rate SAM}
\loss(w_{t}) - \loss(w^\star)\leq   \big (1-  \alpha \eta (2  - \eta  \beta)\big )^t 
(\loss(w_{0}) - \loss(w^\star)) +  \frac{\eta \beta^3\rho^2}{2\alpha(2-\eta\beta)},\quad \forall t.
\end{align}
\item \textbf{USAM.} In contrast, there exists some $\alpha$-strongly-convex and $\beta$-smooth loss $\loss$ such that the USAM with $\eta \in (\nicefrac{2}{(\beta +\rho \beta^2)},\nicefrac 2\beta]$ diverges for all except measure zero initialization $w_0$.
\end{enumerate}
\end{theorem}
As we discussed, it is well-known that GD converges iff $\eta<\nicefrac 2\beta$, and \autoref{thm:Strongly Convex and Smooth} shows that SAM also does not diverge and stays within an $\mathcal O(\sqrt{\eta}\rho)$-neighborhood around the minimum as long as $\eta<\nicefrac 2\beta$.
However, USAM diverges with an even lower learning rate: $\eta>\nicefrac{2}{(\beta+\rho \beta^2)}$ can already make USAM diverge. Intuitively, the larger the value of $\rho$, the easier it is for USAM to diverge.

One may notice that \autoref{eq:convergence rate SAM}, compared to the standard convergence rate of GD,  exhibits an additive  bias term of order $\mathcal O(\eta \rho^2)$. This term arises from the unstable nature of SAM: the perturbation in~\eqref{eq:definition of SAM} (which always has norm $\rho$) prevents SAM from decreasing the loss monotonically. Thus, SAM can only approach a minimum up to a neighborhood.
For this reason, in this paper whenever we say SAM ``finds'' a minimum, we mean its iterates approach and stay within a neighborhood of that minimum.

Due to space limitations, the full proof is postponed to \autoref{sec:appendix_smooth_strongly_convex} and we only outline it here.
\begin{proof}[Proof Sketch]
For SAM, we show an analog to the descent lemma of GD as follows (see \autoref{lem:SAM descent}):
\begin{align}\label{eq:descent lemma SAM}
\loss(w_{t+1}) \leq      \loss(w_t)   - \frac{1}{2}\eta (2  - \eta  \beta)  \norm{\nabla \loss( w_t)}^2 + \frac{\eta^2 \beta^3\rho^2}{2}\,.
\end{align}

By invoking the strong convexity that gives $\loss(w_t)-\loss(w^\ast)\le \frac{1}{2\alpha}\lVert \nabla f(w_t)\rVert^2$, we obtain
\begin{align*}
\loss(w_{t+1}) - \loss(w^\star) \leq     \big (1-  \alpha \eta (2  - \eta  \beta)\big ) (\loss(w_t)   -   \loss(w^\star)) + \frac{\eta^2 \beta^3\rho^2}{2}\,.
\end{align*}

Recursively applying this relation gives the first conclusion.
For USAM, we consider the quadratic loss function same as \citep{bartlett2022dynamics}. Formally, suppose that $\loss(w)=\frac 12 w^\top \Lambda w$ where $\Lambda=\diag(\lambda_1,\lambda_2,\ldots,\lambda_d)$ is a PSD matrix such that $\lambda_1> \lambda_2 \ge \cdots \lambda_d>0$. Let the eigenvectors corresponding to $\lambda_1,\lambda_2,\ldots,\lambda_d$ be $e_1,e_2,\ldots,e_d$, respectively. Then we show the following in \autoref{lem:USAM Strongly Convex and Smooth}: for any $\eta (\lambda_1+\rho \lambda_1^2)>2$ and $\langle w_0,e_1\rangle\ne 0$, USAM must diverge. As $\loss(w)=\frac 12 w^\top \Lambda w$ is $\lambda_1$-smooth and $\lambda_d$-strongly-convex, the second conclusion also follows.
\end{proof}

Intuitively, the difference in stability can be interpreted as follows:
during the early stage of training, $w_t$ and $\nabla \loss (w_t)$ often  have large norms.
The normalization in SAM then makes the ascent step 
$w_t+\rho \frac{\nabla \loss(w_t)}{\lVert \nabla \loss(w_t)\rVert}$ not too far away from $w_t$.
Hence, if GD does not diverge for this $\eta$, SAM does not either (unless the $\rho$-perturbation is non-negligible, i.e., $\lVert w_t\rVert\gg \rho$ no longer holds).
This is not true for USAM: since the ascent step is un-normalized, it leads to a point far away from $w_t$, making the size of USAM updates much larger. 
In other words, the removal of normalization leads to much more aggressive steps, resulting in a different behavior than GD and also an easier divergence.

\subsection{Generalizing to Non-Convex Cases: Scalar Factorization Problem}
\label{sec:square loss}

Now let us move on to non-convex losses. 
We consider a \textit{scalar version} of the matrix factorization problem $\min_{U,V}\frac 12\lVert UV^T-A\rVert_2^2$, whose loss function is defined as $\loss(x,y)=\frac 12(xy)^2$. 
Denote the initialization by $(x_0,y_0)$, then $\loss(x,y)$ is $\beta\triangleq (x_0^2+y_0^2)$-smooth inside the region $\{(x,y):x^2+y^2\le \beta\}$. Hence, a learning rate $\eta< \nicefrac 2 \beta$ again allows GD to converge due to the well-known descent lemma. 
The following result compares the behavior of SAM and USAM under this setup.

\begin{theorem}[Scalar Factorization Problem; Informal]\label{thm:square loss}
For the loss function $\loss(x,y)=\frac 12(xy)^2$ restricted to a $\beta$-smooth region, if we set $\eta=\nicefrac 1\beta<\nicefrac 2\beta$ (so GD finds a minimum), then
\begin{enumerate}[itemsep=2pt]
\item \textbf{SAM.} SAM never diverges and approaches a minimum within an $O(\rho)$-neighborhood (in fact, SAM with distinct $\rho$'s always find the same minimum $(0,0)$).
\item \textbf{USAM.} On the other hand, USAM diverges once $\rho\ge 15 \eta$ -- which could be much smaller than 1.
\end{enumerate}
\end{theorem}

Thus, our observation in \autoref{thm:Strongly Convex and Smooth} is not limited to convex losses -- for our non-convex scalar-factorization problem, the stability of SAM remains robust to the choice of $\rho$, while USAM is provably unstable.
One may refer to \autoref{sec:appendix_square_loss} for the formal statement and proof of \autoref{thm:square loss}.

\subsection{Experiment: Early-Stage Behaviors when Training Neural Networks}
\label{sec:nn_earlystage}
\begin{figure}[H]
\centering
\includegraphics[width=0.7\textwidth]{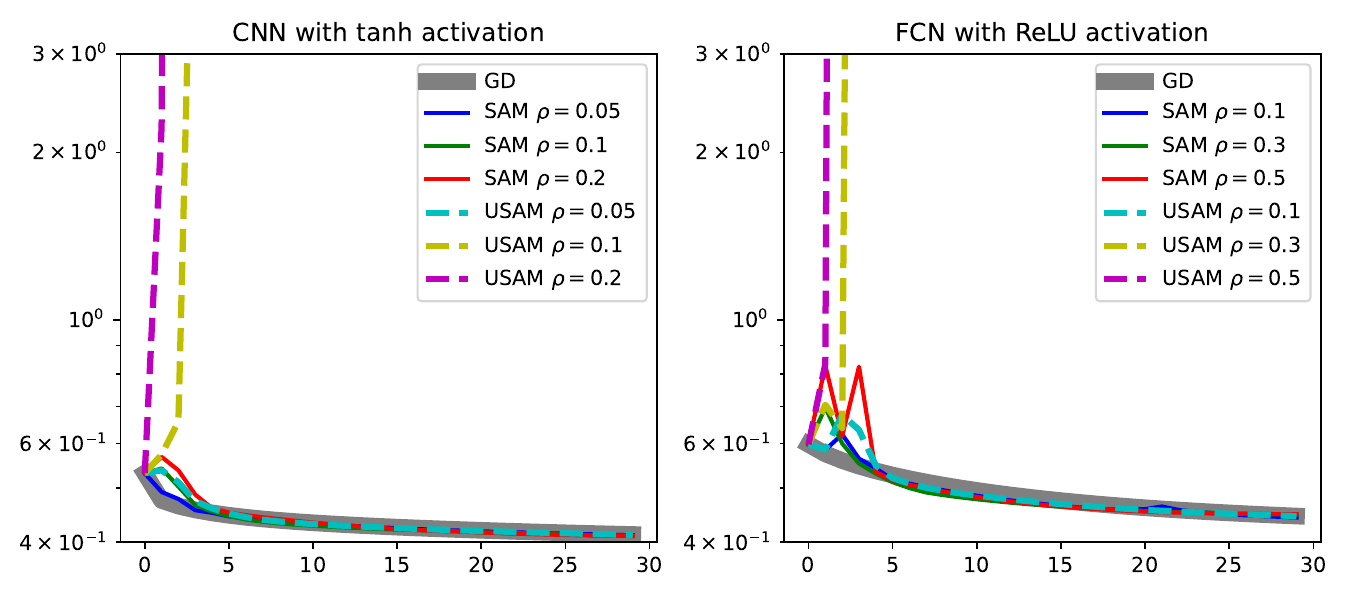}
\caption{{Behaviors of different algorithms when training neural networks} \footnotesize{ ($\eta=0.025$).}}
\label{fig:nn_earlystage}
\end{figure}

As advertised, our result holds not only for convex or toy loss functions but also for practical neural networks.
In \autoref{fig:nn_earlystage}, we plot the early-stage behavior of GD, SAM, and USAM with different $\rho$ values (while fixing $\eta$).
We pick two neural networks: a convolutional neural network (CNN) with tanh activation and a fully-connected network (FCN) with ReLU activation. We train them over the CIFAR10 dataset and report the early-stage training losses.
Similar to \autoref{fig:stability}, \autoref{thm:Strongly Convex and Smooth} and \autoref{thm:square loss}, \textbf{\emph{the stability of SAM is not sensitive to the choice of $\rho$, while USAM diverges easily}}.

\section{Role of Normalization for Drifting Near Minima}
\label{sec:solution}
Now, we explain the second role of normalization: enabling the algorithm to drift near minima.
To convince why this is beneficial, we adopt a loss function recently considered by \citet{ahn2022learning} when understanding the behavior of GD with large learning rates.
Their result suggests that GD needs a ``large enough'' $\eta$ for enhanced performance, but this threshold can never be known a-priori in practice.
To verify our observations from \autoref{fig:solution}, we study the dynamics of SAM and USAM over the same loss function and find that: i) \textbf{\emph{no careful tuning is needed for SAM}}; instead, SAM with any configuration finds the same minimum (which is the ``best'' one according to \citet{ahn2022learning}); and ii)
\textbf{\emph{such property is only enjoyed by SAM}} -- for USAM, careful tuning remains essential.

\subsection{Toy Model: Single-Neuron Linear Network Model}
\label{sec:ell_xy}
To theoretically study the role of normalization near minima, we consider the simple two-dimensional non-convex loss $\loss(x,y)$ defined over all $(x,y)\in\mathbb R^2$ as
\begin{align}\label{eq:definition of ell(xy)}
\loss\colon (x,y) \mapsto \ell(x \times y)\,,\qquad\text{where $\ell$ is convex, even, and 1-Lipschitz.}
\end{align} 
This $\loss$ was recently studied in \citep{ahn2022learning} to understand the behavior of GD with large $\eta$'s.
By direct calculation, the gradient and Hessian of $\loss$ at a given $(x,y)$ can be written as:
\begin{align}
\nabla \loss(x,y)
= \ell'(xy)\, \begin{bmatrix} y \\ x \end{bmatrix},\quad 
\nabla^2 \loss(x,y)
= \ell''(xy) \,\begin{bmatrix} y \\ x \end{bmatrix}^{\otimes 2} + \ell'(xy) \,\begin{bmatrix} 0 & 1 \\ 1 & 0 \end{bmatrix}\,. \label{eq:hessian_lxy}
\end{align}

Without loss of generality, one may assume $\ell$ is minimized at $0$
(see \autoref{sec:appendix_single_neuron} for more details regarding $\ell$).
Then,  $\loss$  achieves minimum at the entire $x$- and $y$-axes, making it a  good toy model for studying the behavior of algorithms near a continuum of minima.
Finally, note that the parametrization $x\times y$ can be interpreted as a single-neuron linear network model -- hence its name.

Before moving on to SAM and USAM we first briefly introduce the behavior of GD on such loss functions characterized in \citep{ahn2022learning}.
Since $\ell$ is even, without loss of generality, we always assume that the initialization $w_0=(x_0,y_0)$ satisfies $y_0\geq x_0>0$.

\begin{theorem}[Theorems 5 and 6 of \citep{ahn2022learning}; Informal]\label{thm:GD ell(xy)}
For any $\eta=\gamma/(y_0^2-x_0^2)$, the GD trajectory over the loss function $\loss(x,y)=\ell(xy)$ has two possible limiting points:
\begin{enumerate}[itemsep=2pt]
\item If $\gamma<2$, then the iterates converge to $(0,y_\infty)$ where $y_\infty^2\in [\nicefrac \gamma \eta-\mathcal O(\gamma)-\mathcal O(\nicefrac \eta \gamma),\nicefrac \gamma \eta+\mathcal O(\nicefrac \eta \gamma)]$.
\item If $\gamma>2$, then the iterates converge to $(0,y_\infty)$ where $y_\infty^2\in [\nicefrac 2 \eta - \mathcal O(\eta),\nicefrac 2\eta]$.
\end{enumerate}
\end{theorem}

Intuitively, the limiting point of GD (denoted by $(0,y_\infty)$) satisfies $y_\infty^2\approx \min\{y_0^2-x_0^2,\nicefrac 2\eta\}$. For simplicity, we denote $\eta_{\text{GD}}\approx \nicefrac{2}{y_0^2-x_0^2}$ as the threshold of $\eta$ that distinguishes these two cases.

\begin{wrapfigure}[20]{R}{0.3\textwidth} \centering
\includegraphics[width=\linewidth]{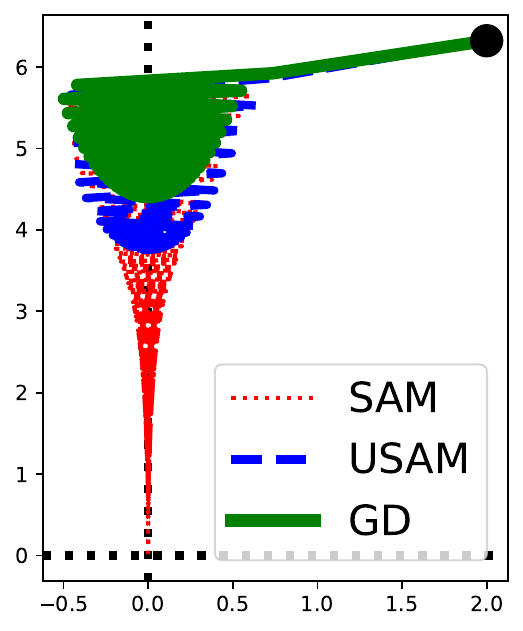}
\caption{ {Trajectories of different algorithms for the $\ell(xy)$ loss}\footnotesize{  ($\eta =0.4$ and $\rho = 0.1$; initialization $(x_0,y_0)=(2, \sqrt{40})$ is marked by a black dot).} }\label{fig:ell(xy)}
\end{wrapfigure}

\textbf{Interpretation of \citet{ahn2022learning}.} Fixing the initialization $(x_0,y_0)$, it turns out this model has a nice connection to the sparse coding problem, wherein it's desirable to get a smaller $y_\infty^2$ (which we will briefly discuss in \autoref{sec:sparse_coding}).
According to \autoref{thm:GD ell(xy)}, to get a smaller $y_\infty^2$,
one must increase the learning rate $\eta$ beyond $\eta_{\text{GD}}$.
Hence we mainly focus on the case where $\eta>\eta_{\text{GD}}$ -- in which case we abbreviate $y_\infty^2\approx \nicefrac 2\eta$ (see \autoref{tab:ell(xy)}).
However, GD diverges once $\eta$ is too large -- in their language, $\gamma$ cannot be much larger than $2$. This dilemma of tuning $\eta$, as we shall illustrate in \autoref{sec:sparse_coding} in more detail, makes GD a brittle choice for obtaining a better $y_\infty^2$.

On the other hand, from the numerical illustrations in \autoref{fig:ell(xy)}, one can see that \textit{SAM keeps moving along the manifold of minimizers} (i.e., the $y$-axis) until the origin. This phenomenon is characterized in \autoref{thm:SAM ell(xy) informal};
in short, any moderate choice of $\eta$ and $\rho$ suffices to drive SAM toward the origin -- no difficult tuning needed anymore!

In contrast, USAM does not keep moving along the axis. Instead, a lower bound on $y_\infty^2$ also presents -- although smaller than the GD version.
As we will justify in \autoref{thm:USAM ell(xy) informal}, \emph{USAM does get trapped} at some non-zero $y_\infty^2$
Thus, a dilemma similar to that of GD shows up: for enhanced performance, an aggressive $(\eta,\rho)$ is needed; however, as we saw from \autoref{sec:stability}, this easily results in a divergence.

\textbf{Assumptions.}
To directly compare with \citep{ahn2022learning}, we focus on the cases where $y_0^2-x_0^2 = \gamma /\eta$ and $\gamma\in [\frac 12,2]$ is a constant of moderate size; hence, $\eta$ is not too different from the $\eta_{\text{GD}}$ defined in \autoref{thm:GD ell(xy)}.
In contrast to most prior works which assume a tiny $\rho$ (e.g., \citep{wen2022does}), we allow $\rho$ to be as large as a constant (i.e., we only require $\rho=\mathcal O(1)$ in \autoref{thm:SAM ell(xy) informal} and \autoref{thm:USAM ell(xy) informal}).

\subsubsection{SAM Keeps Drifting Toward the Origin}
\label{sec:dynamics of SAM}

We characterize the trajectory of SAM when applied to the loss defined in \autoref{eq:definition of ell(xy)} as follows:

\begin{theorem}[SAM over Single-Neuron Networks; Informal]\label{thm:SAM ell(xy) informal}
For any $\eta\in [\frac 12\eta_{\text{GD}},2\eta_{\text{GD}}]$ and $\rho=\mathcal O(1)$, the SAM trajectory over the loss function $\loss(x,y)=\ell(xy)$ can be divided into three phases:
\begin{enumerate}[itemsep=2pt]
\item \textbf{\pone.} $x_t$ drops so rapidly that $\lvert x_t\rvert=\mathcal O(\sqrt \eta)$ in $\mathcal O(\nicefrac 1\eta)$ steps. Meanwhile, $y_t$ remains large: specifically, $y_t=\Omega(\sqrt{\nicefrac 1\eta})$. Thus, SAM approaches the $y$-axis (the set of global minima).
\item \textbf{\ptwo.} $x_t$ oscillates closely to the axis such that $\lvert x_t\rvert=\mathcal O(\sqrt \eta)$ always holds. Meanwhile, $y_t$ decreases fast until $y_t\le \lvert x_t\rvert$\footnote{Specifically,  either an exponential decay $y_{t+1}\lesssim (1-\eta \rho^2)y_t$ or a constant drop $y_{t+1}\lesssim y_t-\eta \rho$ appears\label{footnote:decrease of y_t in SAM}} -- that is, $\lvert x_t\rvert$ remains small and SAM approaches the origin. 
\item \textbf{\pthree.} $w_t=(x_t,y_t)$ gets close to the origin such that $\lvert x_t\rvert,\lvert y_t\rvert=\mathcal O(\sqrt \eta+\eta \rho)$. We then show that $w_t$ remains in this region for the subsequent iterates. 
\end{enumerate}
\end{theorem}

Informally, SAM first approaches the minimizers on $y$-axis (which form a manifold) and then keeps moving until a specific minimum.
Moreover, SAM always approaches this minimum for almost all $(\eta,\rho)$'s.
This matches our motivating experiment in \autoref{fig:solution}: No matter what hyper-parameters are chosen, SAM \textit{always} drift along the set of minima, in contrast to the behavior of GD.  
This property allows SAM always to approach the origin $(0,0)$ and remains in its neighborhood, while GD converges to $(0,\sqrt{\nicefrac 2\eta})$ (see \autoref{tab:ell(xy)}).
The formal version of \autoref{thm:SAM ell(xy) informal} is in \autoref{sec:appendix_sam_ell(xy)}.

\subsubsection{USAM Gets Trapped at Different Minima}
\label{sec:dynamics of USAM}

We move on to characterize the dynamics of USAM near the minima. 
Like GD or SAM, the first few iterations of USAM drive iterates to the $y$-axis.
However, unlike SAM,  USAM does not keep drifting along the $y$-axis and stops at some threshold -- in the result below, we prove a lower bound on $y_t^2$ that depends on both $\eta$ and $\rho$.
In other words,  the lack of normalization factor leads to diminishing drift.

\begin{theorem}[USAM over Single-Neuron Networks; Informal]\label{thm:USAM ell(xy) informal}
For any $\eta\in [\frac 12\eta_{\text{GD}},2\eta_{\text{GD}}]$ and $\rho=\mathcal O(1)$, the USAM trajectory over the loss function $\loss(x,y)=\ell(xy)$ have the following properties:
\begin{enumerate}[itemsep=1pt]
\item \textbf{\pone.} Similar to \pone of \autoref{thm:SAM ell(xy) informal}, $\lvert x_t\rvert=\mathcal O(\sqrt \eta)$ and $y_t=\Omega(\sqrt{\nicefrac 1\eta})$ hold for the first $\mathcal O(\nicefrac 1\eta)$ steps. That is, USAM also approaches $y$-axis, the set of global minima. 
\item \textbf{\pthree.} However, for USAM, once the following condition holds for some round $\tenter$:\footnote{The $\tilde y_{\text{USAM}}^2$ in \autoref{eq:threshold in USAM} is defined as the solution to the equation $(1+\rho y^2)y^2=2$.
\label{footnote:threshold in USAM}}
\begin{equation}\label{eq:threshold in USAM}
(1+\rho y_\tenter^2)y_\tenter^2<\frac 2\eta, \quad \text{i.e.}, \ \ y_\tenter^2<\tilde y_{\text{USAM}}^2\triangleq \left .\left (\sqrt{8\frac{\rho}{\eta}+1}-1\right ) \middle / 2\rho\right .,
\end{equation}
$\lvert x_t\rvert$ decays exponentially fast, which in turn ensures $y_\infty^2\triangleq \liminf_{t\to \infty}y_t^2\gtrsim (1-\eta^2-\rho^2)y_\tenter^2$.
\end{enumerate}
\end{theorem}

\textbf{Remark.}
Note that USAM becomes GD as we send $\rho \to 0+$, and our characterized threshold $y_{\text{USAM}}^2$ indeed recovers
that of GD (i.e., $\nicefrac 2\eta$ from \autoref{thm:GD ell(xy)}) because $\lim_{\rho \to 0+} (\sqrt{8\nicefrac{\rho}{\eta}+1}-1 ) / 2\rho  = \nicefrac 2\eta$.

\begin{table}[tb]
\centering
\renewcommand{\arraystretch}{1.25}
\caption{Limiting points of GD, SAM, and USAM for the $\ell(xy)$ loss {\footnotesize (assuming $\eta>\eta_{\text{GD}}$). }}
\label{tab:ell(xy)}
\begin{tabular}{|c|c|c|c|}\hline
Algorithm     & GD & SAM & USAM  \\\hline
Limiting Point $(0,y_\infty)$ & $y_\infty^2\approx\nicefrac 2\eta$ & $y_\infty^2\approx 0$ & $(1+\rho y_\infty^2)y_\infty^2\approx \nicefrac 2\eta$ \\\hline
\end{tabular}
\end{table}

Compared with SAM, the main difference occurs when close to minima, i.e., $\lvert x_t\rvert=\mathcal O(\sqrt \eta)$.
Consistent with our motivating experiment (\autoref{fig:solution}), the removal of normalization leads to diminishing drift along the minima. 
Thus, USAM is more like an improved version of GD rather than a simplification of SAM,
and the comparison between \autoref{thm:GD ell(xy)} and \autoref{thm:USAM ell(xy) informal} reveals that USAM only improves over GD if $\rho$ is large enough -- in which case USAM is prone to diverges as we discussed in \autoref{sec:stability}. 

See \autoref{sec:appendix_usam_ell(xy)} for a formal version of \autoref{thm:USAM ell(xy) informal} together with its proof.

\subsubsection{Technical Distinctions Between GD, SAM, and USAM}
\label{sec:technical ell(xy)}
Before moving on, we present a more technical comparison between the results stated in \autoref{thm:GD ell(xy)} versus \autoref{thm:SAM ell(xy) informal} and \autoref{thm:USAM ell(xy) informal}.
We start with an intuitive explanation of why GD and USAM get stuck near the manifold of minima but SAM does not:
when the iterates approach the set of minima,  both $w_t$ and $\nabla \loss (w_t)$ become small.
Hence the normalization plays an important role:
as $\nabla \loss(w_t)$ are small,  $w_t$ and $w_t+\rho \nabla \loss(w_t)$ become nearly identical, which leads to a diminishing updates of GD and USAM near the minima. 
On the other hand, having the normalization term, the SAM update doesn't diminish,
which prevents SAM from converging to a minimum and keeps drifting along the manifold. 

This high-level idea is supported by the following calculation: 
recall \autoref{eq:hessian_lxy} that $\nabla \loss(x_t,y_t)=\ell'(x_ty_t) \cdot [y_t\quad x_t]^\top$.
Hence, when $\lvert x_t\rvert\ll y_t$ in \pthree, the ``ascent gradient'' direction $\nabla \loss(x_t,y_t)$ (i.e., the ascent steps in \autoref{eq:definition of SAM} and \autoref{eq:definition of USAM}) is almost perpendicular to the $y$-axis.
We thus rewrite the update direction (i.e., the difference between $w_{t+1}$ and $w_t$) for each algorithm as follows.
\begin{enumerate}[itemsep=2pt]
\item For SAM, after normalization, $\frac{\nabla \loss(w_t)}{\lVert \nabla \loss(w_t)\rVert}$ is roughly a unit vector along the $x$-axis. Hence, the update direction is the gradient at $w_{t+\nicefrac 12}\approx [\rho\quad y_t]^\top$. Once $y_t$ is large (making $w_{t+\nicefrac 12}$ far from minima), $\nabla \loss(w_{t+\nicefrac 12})$ thus have a large component along $y_t$, which leads to drifting near minima. 
\item For GD, by approximating $\ell'(u)\approx u$, we derive $\nabla \loss(x_t,y_t)\approx [x_ty_t^2\quad x_t^2y_t]^\top$. When $\nicefrac 2\eta>y_t^2$, the magnitude of $x_t$ is updated as $\lvert x_{t+1}\rvert\approx \lvert x_t-\eta x_ty_t^2\rvert=\lvert (1-\eta y_t^2)x_t\rvert$, which allows an exponential decay. Thus, GD converges to a minimum and stop moving soon after $\nicefrac 2\eta>y_t^2$.
\item For USAM, the descent gradient is taken at $w_t+\rho \nabla \loss(w_t)\approx [(1+\rho y_t^2)x_t\quad (1+\rho x_t^2)y_t]^\top$. Thus, $\nabla \loss(w_t+\rho \nabla \loss(w_t))\approx [(1+\rho y_t^2)(1+\rho x_t^2)^2 x_ty_t^2\quad
(1+\rho y_t^2)^2(1+\rho x_t^2) x_t^2y_t]^\top$ by writing $\ell'(u)\approx u$. This makes USAM deviate away from SAM and behave like GD: by the similar argument as GD, USAM stops at a minimum soon after $\nicefrac 2\eta>(1+\rho y_t^2)(1+\rho x_t^2)^2 y_t^2\approx (1+\rho y_t^2) y_t^2$!

\end{enumerate}

Hence, the normalization factor in the ascent gradient helps maintain a non-diminishing component along the minima, leading SAM to keep drifting.
This distinguishes SAM from   GD and USAM.

\subsection{USAM Gets Trapped Once Close to Minima}
\label{sec:general}

In this section, we extend our arguments to nonconvex costs satisfying Polyak-Lojasiewicz (PL) functions (see, e.g., \citep{karimi2016linear}).
Recall that $f$ satisfies the $\mu$-PL condition if $\frac 12 \lVert \nabla \loss(w)\rVert^2\ge \mu(\loss(w)-\min_{w}\loss(w))$ for all $w$. 
Building upon the analysis of  \cite{andriushchenko2022towards}, we show the following result when applying USAM to $\beta$-smooth and $\mu$-PL losses.

\begin{theorem}[USAM over PL Losses; Informal]\label{thm:USAM general}
For $\beta$-smooth and $\mu$-PL loss $\loss$, for any $\eta<\nicefrac 1\beta$ and $\rho<\nicefrac 1\beta$, and for any initialization $w_0$, $\lVert w_t-w_0\rVert\le \text{poly}(\eta,\rho,\beta,\mu)\cdot \sqrt{\loss(w_0)-\min_{w}\loss(w)}$, $\forall t$.
\end{theorem}

This theorem has the following consequence: Suppose that USAM encounters a point $w_0$ that is close to some minimum (i.e., $\loss(w_0)\approx \min_{w}\loss(w)$) during training. 
Then \autoref{thm:USAM general} implies that \emph{\textbf{the total distance traveled by USAM from $w_0$ is bounded}} -- in other words, the distance USAM moves along the manifold of minimizers can only be of order $\mathcal O(\sqrt{\loss(w_0)-\min_{w}\loss(w)})$.

As a remark, we compare \autoref{thm:USAM general} with the recent result by \citet{wen2022does}: their result   essentially implies that, for small enough $\eta$ and $\rho$, SAM iterates initialized close to a  manifold of the minimizers  approximately track some continuous dynamics (more precisely, a Riemannian gradient flow induced by a ``sharpness'' measure they find) and keep drifting along the manifold. 
This property is indeed in sharp contrast with USAM whose total travel distance is bounded.  

The formal statement and proof of \autoref{thm:USAM general} are contained in \autoref{sec:appendix_general}.

\subsection{Experiments for Practical Neural Networking Training}
\label{sec:resnet}

\begin{figure}[H]
\centering
\includegraphics[width=0.4\textwidth]{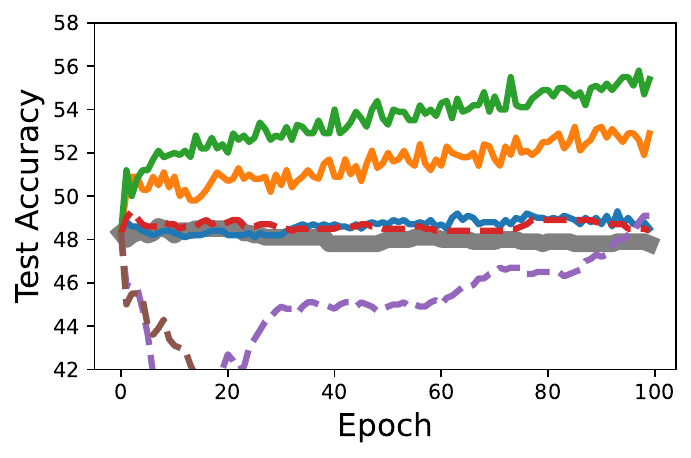}
\includegraphics[width=0.4\textwidth]{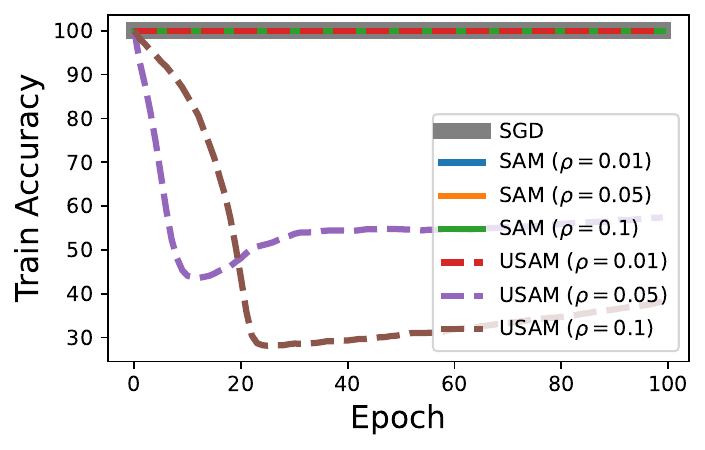}
\caption{{Training ResNet18 on CIFAR-10 from a bad global minimum}\footnotesize{ ($\eta = 0.001$, batch size $=128$).}}
\label{fig:solution resnet}
\end{figure} 

We close this section by verifying our claims in practical neural network training.
We train a ResNet18 on the CIFAR-10 dataset, initialized from a poor global minimum generated as per \citet{liu2020bad} (we used the ``adversarial checkpoint'' released by \citet{damian2021label}). 
This initialization has $100\%$ training accuracy but only $48\%$ test accuracy -- which lets us observe a more pronounced algorithmic behavior near the minima via tracking the test accuracy.
From \autoref{fig:solution resnet}, we observe:
\begin{enumerate}
\item GD gets stuck at this adversarial minimum, in the sense that the test accuracy stays ats $48\%$.
\item SAM keeps drifting while staying close to the manifold of minimizers (because the training accuracy remains $100\%$), which results in better solutions (i.e., the test accuracy keeps increasing).
\item USAM with small $\rho$ gets stuck like GD, while USAM with larger $\rho$'s deviate from this manifold.
\end{enumerate}

Hence, USAM faces the dilemma that we describe in \autoref{sec:ell_xy}:
a conservative hyper-parameter configuration does not lead to much drift along the minima, while a more aggressive choice easily leads to divergence.
However, the stability of SAM is quite robust to the choice of hyper-parameter and they all seem 
to lead to consistent drift along the minima. 

\begin{remark*}
Apart from the ``adversarial checkpoint'' which is unrealistic but can help highlight different algorithms' behavior when they are close to a bad minimum, we also conduct the same experiments but instead initialized from a ``full-batch checkpoint'' \citep{damian2021label}, which is the 100\% training accuracy point reached by running full-batch GD on the training loss function. The result is plotted as \autoref{fig:realistic ckpt} in \autoref{sec:realistic ckpt}.
One can observe that USAM still gets stuck at the ``full-batch checkpoint'', while SAM keeps increasing its test accuracy via drifting along the minima manifold.
\end{remark*}

\section{Case Study: Learning Threshold Neurons for Sparse Coding Problem}
\label{sec:sparse_coding}

To incorporate our two findings into a single example, we consider training one-hidden-layer ReLU networks for the {sparse coding problem}, a setup considered in \citep{ahn2022learning} to study the role of $\eta$ in GD.
Without going into details, the crux of their experiment is to understand how GD with large $\eta$ finds desired structures of the network -- in this specific case, the desired structure is the negative bias in ReLU unit (also widely known as ``thresholding unit/neuron'').
In this section, we evaluate SAM and USAM under the same setup, illustrating the importance of normalization.

\begin{wrapfigure}[12]{R}{0.5\textwidth}
\centering
\includegraphics[width=\linewidth]{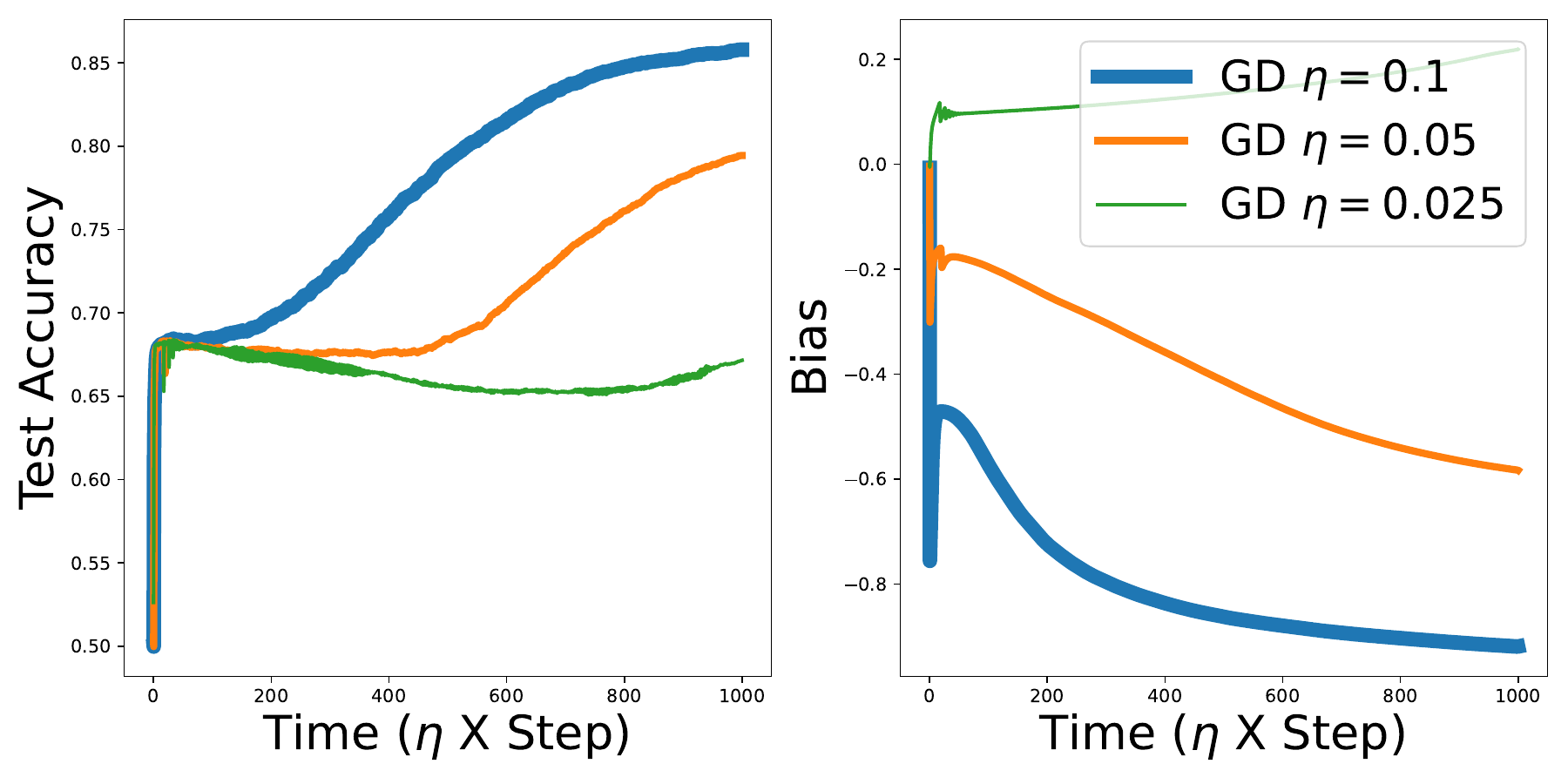}
\caption{{{\footnotesize Behavior of GD for sparse coding problem.}}}
\label{fig:sparse_coding_gd}
\end{wrapfigure}

\textbf{Main observation of \citet{ahn2022learning}.} Given this background, the main observation of \citet{ahn2022learning} is that i) when training the ReLU network with GD, different learning rates induce very different trajectories; and
ii) the desired structure, namely a \textbf{\emph{negative bias in ReLU, only arises with  large ``unstable'' learning rates}} for which GD exhibits unstable behaviors.
We reproduce their results in  \autoref{fig:sparse_coding_gd}, plotting the test accuracy on the left and the bias of ReLU unit on the right. As they claimed, GD with larger $\eta$ learns more negative bias, which leads to better test accuracy.

Their inspiring observation is however a bit discouraging for practitioners.
According to their theoretical results, such learning rates have to be quite large -- large to the point where GD shows very unstable behavior (\`a la Edge-of-Stability~\citep{Cohen2021}).
In practice, without knowledge of the problem, this requires a careful hyper-parameter search to figure out the correct learning rate.
More importantly, such large and unstable learning rates may cause GD to diverge or lead to worse performance. More discussions can be found in the recent paper by \citet{kaur2022maximum}.

In contrast, as we will justify shortly, \textbf{\emph{SAM does not suffer from such a ``dilemma of tuning''}} -- matching with our results in \autoref{thm:SAM ell(xy) informal}.
Moreover, \textbf{\emph{the removal of normalization no longer attains such a property}}, as we demonstrated in \autoref{thm:USAM ell(xy) informal}.
In particular, for USAM, one also needs to carefully tune $\eta$ and $\rho$ for better performance -- as we inspired in \autoref{thm:USAM ell(xy) informal} and \autoref{thm:USAM general}, small $(\eta,\rho)$ makes the iterates get stuck early; on the other hand, as we presented in \autoref{sec:stability}, an aggressive choice causes USAM to diverge.
The following experiments illustrate these claims in more detail.

\begin{figure}[htb]
\centering
\includegraphics[width=0.9
\textwidth]{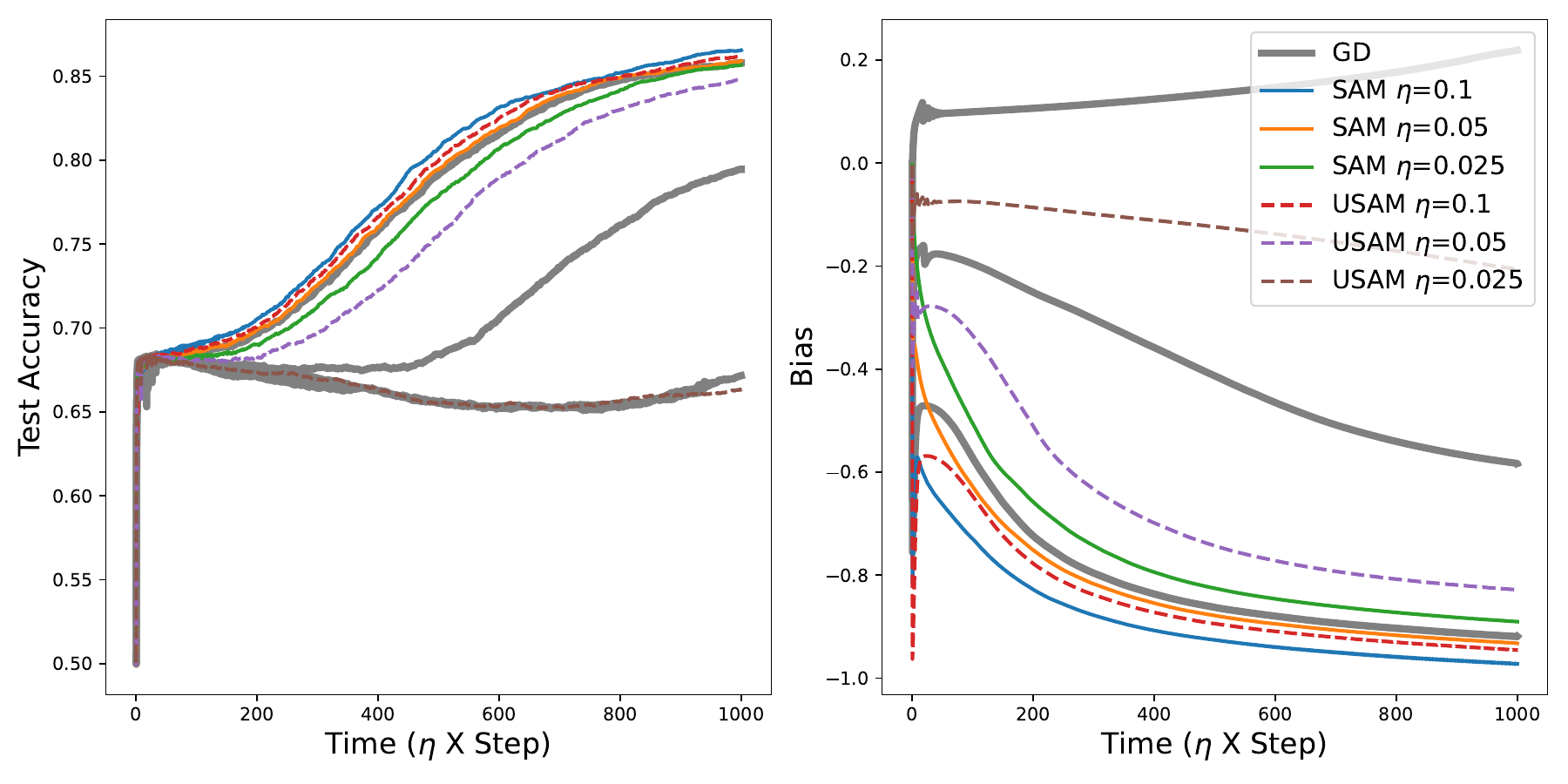}
\caption{{Behaviors of different algorithms for sparse coding problem ($\rho = 0.01$).} {The gray curves (corresponding to GD with different learning rates) are taken from \autoref{fig:sparse_coding_gd} with the same set of $\eta$'s.}}
\label{fig:sparse coding}
\end{figure}

In \autoref{fig:sparse coding}, we plot the performance of SAM, USAM, and GD with different $\eta$'s (while fixing $\rho$) -- gray lines for GD, solid lines for SAM, and dashed lines for USAM.
From the plot, USAM behaves more similarly to GD than SAM: the bias does not decrease sufficiently when the learning rate is not large enough, which consequently to leads poor test accuracy.
On the other hand, no matter what $\eta$ is chosen for SAM, bias is negative enough and ensures better generalization.
Hence, \autoref{fig:sparse coding} illustrates that compared to SAM, USAM is less robust to the tuning of $\eta$.

In \autoref{fig:sparse_coding_rho} (deferred to \autoref{sec:appendix_more_exp}), we also compare these three algorithms when varying $\rho$ and fixing $\eta$.
In addition to what we observe in \autoref{fig:sparse coding}, we show that normalization also helps stability -- USAM quickly diverges as we increase $\rho$, while SAM remains robust to the choice of $\rho$.
Thus, USAM is also less robust to the tuning of $\rho$. In other words, our observation in \autoref{fig:sparse coding} extends to $\rho$.

Hence, putting \autoref{fig:sparse coding} and \autoref{fig:sparse_coding_rho} together, we conclude that \emph{\textbf{SAM is robust to different configurations of $(\eta,\rho)$ while USAM is robust to neither of them}}.
Hence, the normalization of SAM eases hyper-parameter tuning, which is typically a tough problem for GD and many other algorithms -- normalization boosts the success of SAM in practice.

\section{Conclusion}
In this paper, we investigate the role played by normalization in SAM.
By theoretically characterizing the behavior of SAM and USAM on both convex and non-convex losses and empirically verifying our conclusions via real-world neural networks, we found that normalization i) helps stabilize the algorithm iterates, and ii) enables the algorithm to keep moving along the manifold of minimizers, leading to better performance in many cases. 
Moreover, as we demonstrate via various experiments, these two properties make SAM require less hyper-parameter tuning, supporting its practicality. 

In this work, we follow a recent research paradigm of ``physics-style'' approaches to understanding deep neural networks based on simplified models (c.f. \citep{zhang2022unveiling,garg2022can,von2022transformers,abernethy2023mechanism,allen2023physics, liu2022transformers,li2023transformers,ahn2022learning,ahn2023transformers,ahn2023linear}).
We found such physics-style approaches quite helpful, especially for complex modern neural networks.
We hope that our work builds stepping stones for future works on understanding working mechanisms of modern deep neural networks. 

\section*{Acknowledgments}
Kwangjun Ahn was supported by the ONR grant (N00014-20-1-2394) and MIT-IBM Watson as well
as a Vannevar Bush fellowship from Office of the Secretary of Defense. Suvrit Sra acknowledges
support from an NSF CAREER grant (1846088), and NSF CCF-2112665 (TILOS AI Research
Institute). Kwangjun Ahn also acknowledges support from the Kwanjeong Educational Foundation.
We thank Kaiyue Wen, Zhiyuan Li, and Hadi Daneshmand for their insightful discussions.

\bibliographystyle{plainnat}
\bibliography{ref}  

\newpage

\appendix
\renewcommand{\appendixpagename}{\centering \LARGE Appendix}
\appendixpage

\startcontents[section]
\printcontents[section]{l}{1}{\setcounter{tocdepth}{2}}

\newpage
\section{Setup of the Motivating Experiment}
\label{app:exp_detail}
In the motivating experiments (\autoref{fig:stability} and \autoref{fig:solution}), we follow the over-parameterized matrix sensing setup as \citet{li2018algorithmic} and \citet{blanc2020implicit}. Specifically, we do the following:
\begin{enumerate}
\item Generate the true matrix by sampling each entry of ${U^\star}\in \R^{d\times r}$ independently from a standard Gaussian distribution and let ${X^\star}= {U^\star}  ({U^\star})^\top$. 
\item Normalize each column of $ {U^\star}$ to unit norm so that the spectral norm of $ {U^\star}$ is close to one.
\item For every sensing matrix $A_i$ ($i = 1,2,\ldots,m$), sample the entries of $A_i$ independently from a standard Gaussian distribution. Then observe $b_i = \langle A_i,X^\star\rangle$.
\end{enumerate}
In particular, for the experiments, we chose $r = 5$, $d = 100$, and  $m = 5dr$.

\section{Additional Experimental Results}
\subsection{Running SAM and USAM from Other Initializations}
\label{sec:realistic ckpt}
\begin{figure}[H]
\centering
\includegraphics[width=0.8\textwidth]{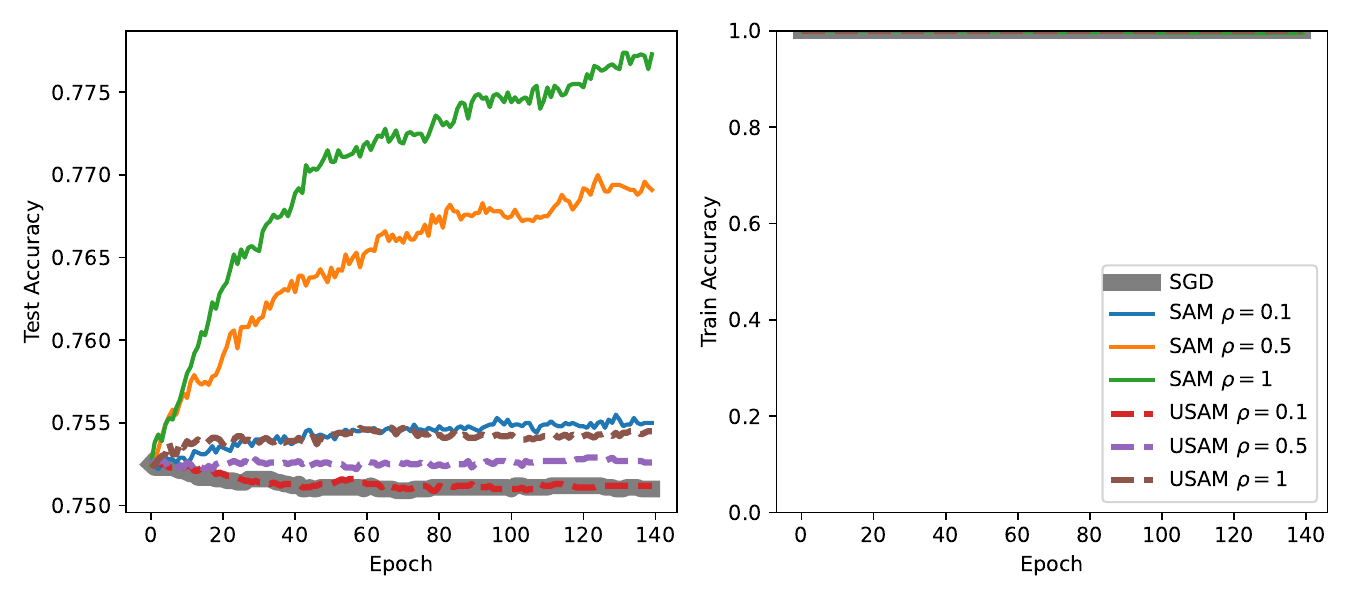}
\caption{{Training ResNet18 on CIFAR-10 from a ``full-batch checkpoint'' of \citet{damian2021label}}\footnotesize{ ($\eta = 0.001$, batch size $=128$).}}
\label{fig:realistic ckpt}
\end{figure} 

\subsection{Varying $\rho$ While Fixing $\eta$ in Sparse Coding Example}
\label{sec:appendix_more_exp}
\begin{figure}[H]
\centering
\includegraphics[width=0.8\textwidth]{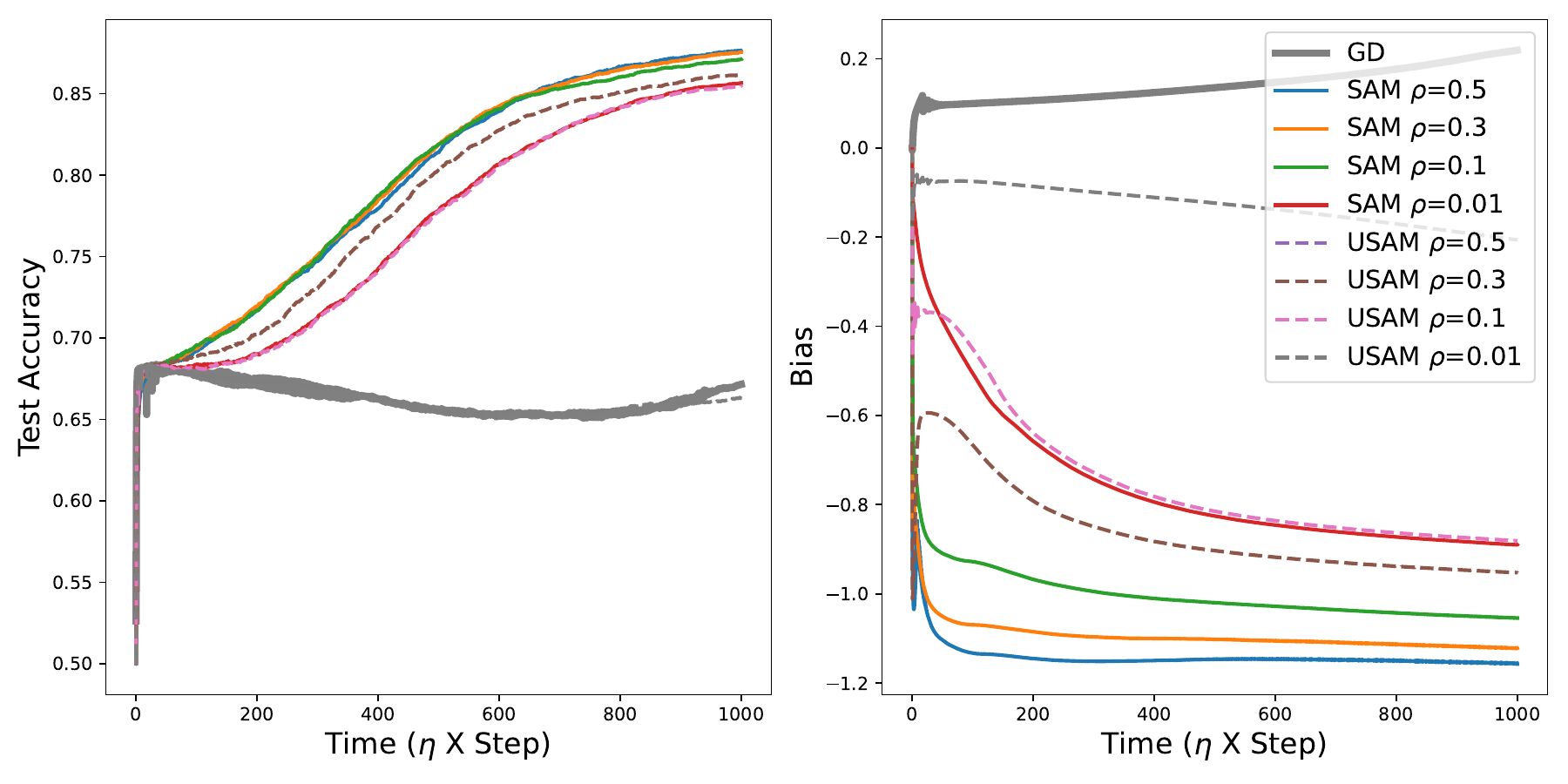}
\caption{{{Behaviors of different algorithms for sparse coding problem} 
($\eta = 0.025$). Note that USAM with $\rho=0.5$ (the dashed violet curve) diverges and becomes invisible except for the very beginning.}}
\label{fig:sparse_coding_rho}
\end{figure}

Aside from \autoref{fig:sparse coding} which varies $\eta$ while fixing $\rho$, we perform the same experiment when varying $\rho$ and fixing $\eta$.
The main observation for SAM is similar to that of \autoref{fig:sparse coding}: different hyper-parameters all keep decreasing the bias and give better test accuracy -- even with the tiniest choice $\rho=0.01$.

However, for USAM, there are three different types of $\rho$'s as shown in \autoref{fig:sparse_coding_rho}:
\begin{enumerate}
\item For tiny $\rho=0.01$, the bias doesn't decrease much. Consequently, the performance of USAM nearly degenerates to that of GD -- while SAM with $\rho=0.01$ still gives outstanding performance.
\item For moderate $\rho=0.1$ and $\rho=0.3$, USAM manages to decrease the bias and improve its accuracy, though with a slower speed than SAM with the same choices of $\rho$.
\item For large $\rho=0.5$ (where SAM still works well; see the solid curve in blue), USAM diverges.
\end{enumerate}

Thus, the dilemma described in \autoref{sec:sparse_coding} indeed applies to not only $\eta$ but also $\rho$ -- matching our main conclusion that normalization helps make hyper-parameter tuning much more manageable.

\section{Omitted Proof for Smooth and Strongly Convex Losses}
\label{sec:appendix_smooth_strongly_convex}

We shall first restate \autoref{thm:Strongly Convex and Smooth} here for the ease of presentation.
\begin{theorem}[Restatement of \autoref{thm:Strongly Convex and Smooth}]
For any $\alpha$-strongly-convex and $\beta$-smooth loss function $\loss$, for any learning rate $\eta \leq \nicefrac 2\beta$ and perturbation radius $\rho\geq 0$, the following holds:
\begin{enumerate}
\item \textbf{SAM.} The iterate $w_t$  converges to a local neighborhood around the minimizer $w^\star$. Formally,
\begin{align*}
\loss(w_{t}) - \loss(w^\star)\leq   \big (1-  \alpha \eta (2  - \eta  \beta)\big )^t 
(\loss(w_{0}) - \loss(w^\star)) +  \frac{\eta \beta^3\rho^2}{2\alpha(2-\eta\beta)},\quad \forall t.
\end{align*}
\item \textbf{USAM.} In contrast, there exists some $\alpha$-strongly-convex and $\beta$-smooth loss $\loss$ such that the USAM with $\eta \in (\nicefrac{2}{(\beta +\rho \beta^2)},\nicefrac 2\beta]$ diverges for all except measure zero initialization $w_0$.
\end{enumerate}
\end{theorem}

We will show these two conclusions separately. The proof directly follows from \autoref{lem:SAM Strongly Convex and Smooth} and \autoref{lem:USAM Strongly Convex and Smooth}.

\subsection{SAM Allows a Descent Lemma Like GD}
\begin{theorem}[SAM over Strongly Convex and Smooth Losses]\label{lem:SAM Strongly Convex and Smooth}
For any $\alpha$-strongly-convex and $\beta$-smooth loss function $\loss$, for any learning rate $\eta \leq \nicefrac 2\beta$ and perturbation radius $\rho\geq 0$, the following holds for SAM:
\begin{align*}
\loss(w_{t}) - \loss(w^\star)\leq   (1-  \alpha \eta (2  - \eta  \beta))^t 
(\loss(w_{0}) - \loss(w^\star)) +  \frac{\eta \beta^3\rho^2}{2\alpha(2-\eta\beta)}\,.
\end{align*}
\end{theorem}
\begin{proof}
We first claim the following analog of descent lemma, which we state as \autoref{lem:SAM descent}:
\begin{align*}
\loss(w_{t+1}) \leq      \loss(w_t)   - \frac{1}{2}\eta (2  - \eta  \beta)  \norm{\nabla \loss( w_t)}^2 + \frac{\eta^2 \beta^3\rho^2}{2}\,.
\end{align*}
By definition of strong convexity, we have
\begin{align*}
\loss(w_t) -\loss(w^\star) \leq \inp{\nabla \loss (w_t)}{w_t - w^\star} - \frac{\alpha}{2}\norm{w_t - w^\star}^2  \leq \frac{1}{2\alpha} \norm{\nabla f(w_t)}^2\,,
\end{align*}
where the last inequality uses the fact that $\inp{a}{b} -\frac{1}{2}\norm{b}^2 \leq \frac{1}{2}\norm{a}^2$.
Thus, combining the two inequalities above, we obtain  
\begin{align*}
\loss(w_{t+1}) \leq      \loss(w_t)   - \alpha \eta (2  - \eta  \beta)   ( \loss(w_t) -\loss(w^\star)) + \frac{\eta^2 \beta^3\rho^2}{2}\,,
\end{align*}
which after rearrangement becomes
\begin{align*}
\loss(w_{t+1}) - \loss(w^\star) \leq     (1-  \alpha \eta (2  - \eta  \beta)) (\loss(w_t)   -   \loss(w^\star)) + \frac{\eta^2 \beta^3\rho^2}{2}\,.
\end{align*}
Unrolling this recursion, we obtain
\begin{align*}
\loss(w_{t}) - \loss(w^\star) &\leq   (1-  \alpha \eta (2  - \eta  \beta))^t 
(\loss(w_{0}) - \loss(w^\star)) +  \frac{\eta^2 \beta^3\rho^2}{2} \sum_{k=0}^{t-1} (1-\alpha\eta(2-\eta\beta))^k\\
&\leq   (1-  \alpha \eta (2  - \eta  \beta))^t 
(\loss(w_{0}) - \loss(w^\star)) +  \frac{\eta^2 \beta^3\rho^2}{2} \sum_{k=0}^{\infty} (1-\alpha\eta(2-\eta\beta))^k\\
&= (1-  \alpha \eta (2  - \eta  \beta))^t 
(\loss(w_{0}) - \loss(w^\star)) +  \frac{\eta^2 \beta^3\rho^2}{2\alpha\eta(2-\eta\beta)} \,.
\end{align*}
This completes the proof.
\end{proof} 

\begin{lemma}[SAM Descent lemma] \label{lem:SAM descent}
For a convex loss $\loss$ that is $\beta$-smooth, SAM iterates $w_t$ satisfy the following when the learning rate $\eta<\frac 2\beta$ and $\rho\ge 0$:
\begin{align*}
\loss(w_{t+1}) \leq      \loss(w_t)   - \frac{1}{2}\eta (2  - \eta  \beta)  \norm{\nabla \loss( w_t)}^2 + \frac{\eta^2 \beta^3\rho^2}{2},\quad \forall t.
\end{align*}
\end{lemma}
\begin{proof}
Let $v_t \triangleq \nabla \loss(w_t)/\norm{\nabla \loss(w_t)}$ and  $w_{t+\nicefrac 12} \triangleq w_t + \rho v_t$ so we have $w_{t+1}=w_t-\eta \nabla \loss(w_{t+\nicefrac 12})$. Then the $\beta$-smoothness of $\loss$ yields
\begin{align*}
\loss(w_{t+1}) \leq \loss(w_t) - \eta \langle  \nabla \loss(w_t),  \nabla \loss( w_{t+\nicefrac 12}  )  \rangle + \frac{\eta^2 \beta}{2} \| \nabla \loss(w_{t+\nicefrac 12}) \|^2.
\end{align*}
We start with upper bounding the norm of $\nabla \loss(w_{t+\nicefrac 12})$:
\begin{align*}
\| \nabla \loss(w_{t+\nicefrac 12}) \|^2  &=    \| \nabla \loss(w_{t+\nicefrac 12}) - \nabla \loss(w_t) \|^2  - \| \nabla \loss( w_t) \|^2   +2 \langle \nabla \loss( w_{t+\nicefrac 12}  ) , \nabla \loss(w_t) \rangle\\
&\leq  \beta^2 \rho^2  - \| \nabla \loss( w_t) \|^2   +2 \langle \nabla \loss( w_{t+\nicefrac 12}  ) , \nabla \loss(w_t) \rangle  ,
\end{align*}  
Hence, as long as $\eta < \frac{2}{\beta}$, we have the following upper bound on $\loss(w_{t+1})$:
\begin{align*}
\loss(w_{t+1}) &\leq \loss(w_t) - \eta \langle \nabla \loss( w_{t+\nicefrac 12}  ) , \nabla \loss(w_t) \rangle + \frac{\eta^2 \beta}{2} \| \nabla \loss(w_{t+\nicefrac 12}) \|^2\\
&\leq \loss(w_t) - \eta \langle \nabla \loss( w_{t+\nicefrac 12}  ) , \nabla \loss(w_t) \rangle + \frac{\eta^2 \beta}{2}  \left(\beta \rho^2  - \| \nabla \loss( w_t) \|^2   +2 \langle \nabla \loss( w_{t+\nicefrac 12}  ) , \nabla \loss(w_t) \rangle  \right)\\
&= \loss(w_t) -  \frac{\eta^2 \beta}{2} \norm{\nabla \loss( w_t)}^2 - (\eta  - \eta^2 \beta) \langle \nabla \loss( w_{t+\nicefrac 12}  ) , \nabla \loss(w_t) \rangle + \frac{\eta^2 \beta^3\rho^2}{2}.
\end{align*}

Now we lower bound $\inp{\nabla \loss( w_{t+\nicefrac{1}{2}})}{ \nabla \loss(w_t) }$. 
Note that 
\begin{align*}
\langle  \nabla \loss( w_t + \rho v_t), \nabla \loss(w_t) \rangle & =   \langle  \nabla \loss( w_t + \rho v_t)-  \nabla \loss(w_t) , \nabla \loss(w_t) \rangle  +   \|  \nabla \loss(w_t) \|^2\\
&= \frac{\norm{\nabla \loss(w_t) }}{\rho }  \langle  \nabla \loss( w_t + \rho v_t )-  \nabla \loss(w_t) , \rho v_t \rangle +  \|  \nabla \loss(w_t) \|^2  \geq    \|  \nabla \loss(w_t) \|^2  \, ,
\end{align*}
where the last inequality uses the following standard fact about convex functions: for any $w_1,w_2$, $\inp{\nabla \loss(w_1) -\nabla \loss(w_2) }{ w_1-w_2 }  \geq 0$. Hence, we arrive at
\begin{align*}
\loss(w_{t+1}) &\leq \loss(w_t) -  \frac{\eta^2 \beta}{2} \norm{\nabla \loss( w_t)}^2 - (\eta  - \eta^2 \beta) \langle \nabla \loss( w_{t+\nicefrac 12}  ) , \nabla \loss(w_t) \rangle + \frac{\eta^2 \beta^3\rho^2}{2}\\
&\leq    \loss(w_t) -  \frac{\eta^2 \beta}{2} \norm{\nabla \loss( w_t)}^2 - (\eta  - \eta^2 \beta)  \norm{\nabla \loss( w_t)}^2 + \frac{\eta^2 \beta^3\rho^2}{2}  \\
&\leq      \loss(w_t)   - \frac{1}{2}\eta (2  - \eta  \beta)  \norm{\nabla \loss( w_t)}^2 + \frac{\eta^2 \beta^3\rho^2}{2} \,, 
\end{align*}
and thus finishing the proof.
\end{proof}

\subsection{USAM Diverges on Quadratic Losses}
\begin{theorem}[USAM over Quadratic Losses]\label{lem:USAM Strongly Convex and Smooth}
Following \citep{bartlett2022dynamics}, consider the quadratic loss $\loss$ induced by a PSD matrix $\Lambda$. Without loss of generality, let $\loss$ minimize at the origin. Formally,
\begin{align} \label{def:quad}
\loss(w) = \frac{1}{2}w^\top\Lambda w\,,\quad \text{where}~\Lambda=\diag(\lambda_1,\ldots,\lambda_d)\,, 
\end{align} 
where $\lambda_{\max}=\lambda_1>\cdots\ge\lambda_d > 0$ and $v_{\max}=e_1,e_2,\dots,e_d$ are the eigenvectors corresponding to $\lambda_1,\lambda_2,\ldots,\lambda_d$, respectively.
Then the iterates of USAM \autoref{eq:definition of USAM} applied to \autoref{def:quad} satisfy:
\begin{enumerate}
\item If $\eta (\lambda_{\max} + \rho \lambda_{\max}^2) < 2$, then the iterates converges to the global minima exponentially fast.
\item If $\eta (\lambda_{\max} + \rho \lambda_{\max}^2) > 2$ and if $\inp{w_0}{v_{\max}}\neq 0$, then the iterates diverge.
\end{enumerate}

Moreover, such a loss function $\loss$ is $\lambda_1$-smooth and $\lambda_d$-strongly-convex.
\end{theorem}
\begin{proof}
As $\nabla\ell(w) = \Lambda w$, the  USAM update \autoref{eq:definition of USAM} reads
\begin{align*}
w_{t+1}
& = w_t - \eta\nabla\loss\left(w_t+\rho \nabla\loss(w_t)\right)  = \left(I - \eta\Lambda - \eta \rho  \Lambda^2\right)w_t. 
\end{align*}
Hence, if $\eta (\lambda_{\max} + \rho \lambda_{\max}^2) < 2$, then since $\norm{I - \eta\Lambda - \eta \rho  \Lambda^2}<1$, we have that $\norm{w_t}\to 0$ exponentially fast.
On the other hand, if $\eta (\lambda_{\max} + \rho \lambda_{\max}^2)>2$,  then  $|1-\eta\lambda_1 - \eta \lambda_1^2|>1$. Since $\norm{w_{t}} \geq |1-\eta\lambda_1 - \eta \lambda_1^2|^t |\inp{w_0}{e_1}|$, it follows that the iterates diverge as long as $\inp{w_0}{v_{\max}}\neq 0$.
Finally, as $\nabla^2 \loss(w)=\Lambda$, we directly know that $\loss$ is $\lambda_1$-smooth and $\lambda_d$-strongly-convex.
\end{proof} 

\section{Omitted Proof for Scalar Factorization Problems}
\label{sec:appendix_square_loss}

\begin{theorem}[Formal Version of \autoref{thm:square loss}]\label{thm:square loss formal}
Consider the scalar factorization loss $\loss(x,y)=\frac 12(xy)^2$. If $(x_0,y_0)$ satisfies $2x_0>y_0>x_0\gg 0$ and $\eta=(x_0^2+y_0^2)^{-1}$, then i) SAM finds a neighborhood of the origin with radius $\mathcal O(\rho)$ for all $\rho$, and ii) USAM diverges as long as $\rho\ge 15\eta$.
\end{theorem}
This theorem is a combination of \autoref{lem:square loss SAM} and \autoref{lem:square loss USAM} that we will show shortly.

\subsection{SAM Always Converges on Scalar Factorization Problems}

Let $w_t=[x_t\quad y_t]^\top$ be the $t$-th iterate. Denote $\nabla_t=\nabla \loss(w_t)=[x_ty_t^2 \quad x_t^2 y_t]^\top$. For each step $t$, the actual update of SAM is therefore the gradient taken at
\begin{equation*}
\tilde w_t=w_t+\rho \frac{\nabla_t}{\lVert \nabla_t\rVert}=\begin{bmatrix}x_t+\rho y_t z_t^{-1}\\y_t+\rho x_t z_t^{-1}\end{bmatrix},\quad \text{where } z_t^{-1}=\frac{\sgn(x_ty_t)}{\sqrt{x_t^2+y_t^2}}.
\end{equation*}

By denoting $\tilde \nabla_t=\nabla \loss(\tilde w_t)$, the update rule of SAM is
\begin{equation*}
\begin{bmatrix}
x_{t+1} \\ y_{t+1}
\end{bmatrix}=w_{t+1}=w_t-\eta \tilde \nabla_t=
\begin{bmatrix}
x_t-\eta \left (x_t+\rho y_tz_t^{-1}\right ) \left (y_t+\rho x_tz_t^{-1}\right )^2 \\ y_t-\eta \left (x_t+\rho y_tz_t^{-1}\right )^2 \left (y_t+\rho x_tz_t^{-1}\right )
\end{bmatrix}.
\end{equation*}

We make the following claim:
\begin{theorem}[SAM over Scalar Factorization Problems]\label{lem:square loss SAM}
Under the setting of \autoref{thm:square loss formal}, there exists some threshold $T$ for SAM such that $\lvert x_t\rvert,\lvert y_t\rvert\le 5\rho$ for all $t\ge T$.
\end{theorem}
\begin{proof}
We first observe that SAM over $\loss(x,y)$ always pushes the iterate towards the minima (which are all the points on the $x$- and $y$- axes). Formally:
\begin{itemize}
\item If $x_t\ge 0$, then $x_{t+1}\le x_t$. If $x_t\le 0$, then $x_{t+1}\ge x_t$.
\item If $y_t\ge 0$, then $y_{t+1}\le y_t$. If $y_t\le 0$, then $y_{t+1}\ge y_t$.
\end{itemize}
This observation can be verified by noticing $\sgn(\rho y_tz_t^{-1})=\sgn(y_t)\sgn(x_ty_t)=\sgn(x_t)$ and similarly $\sgn(\rho x_tz_t^{-1})=\sgn(y_t)$. In other words, $\sgn(x_t+\rho y_tz_t^{-1})=\sgn(x_t)$ and thus always pushes $x_t$ towards the $y$-axis. The same also holds for $y_t$ by symmetry.

In analog to the descent lemma for GD, we can show the following lemma:
\begin{lemma}
When picking $\eta=(x_0^2+y_0^2)^{-1}$, we have $\loss(w_{t+1})-\loss(w_t)<0$ as long as $x_t,y_t\ge 5\rho$.
\end{lemma}
\begin{proof}
As $\nabla^2 \loss(x,y)=\begin{bmatrix}y^2 & 2xy \\ 2xy & x^2\end{bmatrix}$, we know $\loss$ is $\beta\triangleq (x_0^2+y_0^2)$-smooth inside the region $\{(x,y):x^2+y^2\le \beta\}$.
Then we have (recall that $\eta=(x_0^2+y_0^2)^{-1}=\nicefrac 1\beta$)
\begin{align*}
\loss(w_{t+1})-\loss(w_t)&\le \langle \nabla \loss(w_t),w_{t+1}-w_t\rangle+\frac \beta2 \lVert w_{t+1}-w_t\rVert^2\\
&=-\eta \langle \nabla_t,\tilde \nabla_t\rangle+\frac \beta2\eta^2\lVert \tilde \nabla_t\rVert^2=-\eta \left (\langle \nabla_t-\tfrac 12 \tilde \nabla_t,\tilde \nabla_t\rangle\right ).
\end{align*}

To make sure that it's negative, we simply want
\begin{align*}
0\le \langle \nabla_t-\tfrac 12 \tilde \nabla_t,\tilde \nabla_t\rangle&=\left (x_ty_t^2-\frac 12\left (x_t+\rho y_tz_t^{-1}\right ) \left (y_t+\rho x_tz_t^{-1}\right )^2\right )\left (x_t+\rho y_tz_t^{-1}\right ) \left (y_t+\rho x_tz_t^{-1}\right )^2+\\&\quad \left (x_t^2y_t-\frac 12\left (x_t+\rho y_tz_t^{-1}\right )^2 \left (y_t+\rho x_tz_t^{-1}\right )\right )\left (x_t+\rho y_tz_t^{-1}\right )^2 \left (y_t+\rho x_tz_t^{-1}\right ),
\end{align*}
which can be ensured once
\begin{align*}
\left (x_t+\rho y_tz_t^{-1}\right ) \left (y_t+\rho x_tz_t^{-1}\right )^2\le 2x_ty_t^2,\quad
\left (x_t+\rho y_tz_t^{-1}\right )^2 \left (y_t+\rho x_tz_t^{-1}\right )\le 2x_t^2y_t.
\end{align*}

If we have $x_t,y_t\ge 5\rho$, then as $z_t^{-1}\le \min\{x_t^{-1},y_t^{-1}\}$, we have
\begin{equation*}
\left (x_t+\rho y_tz_t^{-1}\right ) \left (y_t+\rho x_tz_t^{-1}\right )^2\le (x_t+\rho)(y_t+\rho)^2\le 1.2^3 x_ty_t^2<2x_ty_t^2,
\end{equation*}
which shows the first inequality. The second one follows from symmetry.
\end{proof}

Therefore, SAM will progress until $x_t\le 5\rho$ or $y_t\le 5\rho$. Without loss of generality, assume that $x_t\le 5\rho$; we then claim that $y_t$ will soon decrease to $\mathcal O(\rho)$.
\begin{lemma}
Suppose that $\lvert x_t\rvert\le 5\rho$ but $y_t\ge 5\rho$. Then $\lvert x_{t+1}\rvert\le 5\rho$ but $y_{t+1}\le (1-\frac 12 \eta \rho^2)y_t$.
\end{lemma}
\begin{proof}
First, show that $\lvert x_t\rvert$ remains bounded by $5\rho$. Assume $x_t\ge 0$ without loss of generality:
\begin{align}
x_{t+1}&=x_t-\eta \left (x_t+\rho y_t z_t^{-1}\right )\left (y_t+\rho x_t z_t^{-1}\right )^2 \nonumber\\
&\ge x_t-\eta \left (5\rho+\rho\right )\left (y_t+\rho\right )^2 \nonumber\\
&\ge -\eta 6\rho \left (\frac{4}{5\eta}+\rho^2+2\rho \sqrt{\frac{4}{5\eta}}\right )\ge -5\rho,\label{eq:square loss x_t lower bound}
\end{align}
where the second last line uses $y_t^2\le y_0^2=\frac{y_0^2}{x_0^2+y_0^2}\eta^{-1}\le \frac 45\eta^{-1}$ and the last one uses $\eta=(x_0^2+y_0^2)^{-1}\ll 1$.
Meanwhile, we see that $y_t$ decreases exponentially fast by observing the following:
\begin{align*}
y_{t+1}&=y_t-\eta (y_t+\rho x_tz_t^{-1})(x_t+\rho y_t z_t^{-1})^2\\
&\le y_t-\eta y_t (\rho y_tz_t^{-1})^2\\
&\le \left (1-\frac 12 \eta \rho^2\right )y_t,
\end{align*}
where the last line uses $z_t^{-1}=(x_t^2+y_t^2)^{-1/2}\ge (2y_t^2)^{-1/2}=2^{-1/2} y_t^{-1}$ as $y_t\ge 5\rho\ge \lvert x_t\rvert$.
\end{proof}

So eventually we have $\lvert x_t\rvert,\lvert y_t\rvert\le 5\rho$. Recall that \autoref{eq:square loss x_t lower bound} infers $\lvert x_{t+1}\rvert\le 5\rho$ from $\lvert x_t\rvert\le 5\rho$.
Hence, by symmetry, we conclude that $\lvert x_{t+1}\rvert,\lvert y_{t+1}\rvert\le 5\rho$ hold as well.
Therefore, SAM always finds an $\mathcal O(\rho)$-neighborhood of the origin, i.e., it is guaranteed to converge regardless of $\rho$.
\end{proof}

\subsection{USAM Diverges with Small $\rho$}
For USAM, the dynamics can be written as
\begin{equation}\label{eq:square loss USAM}
\begin{bmatrix}
x_{t+1} \\ y_{t+1}
\end{bmatrix}=\begin{bmatrix}
x_t \\ y_t
\end{bmatrix}-\eta \nabla \loss \begin{pmatrix}
x_t+\rho x_t y_t^2 \\
y_t+\rho x_t^2 y_t
\end{pmatrix}=\begin{bmatrix}
x_t - \eta (x_t+\rho x_t y_t^2) (y_t+\rho x_t^2 y_t)^2 \\
y_t - \eta (x_t+\rho x_t y_t^2)^2 (y_t+\rho x_t^2 y_t)
\end{bmatrix}.
\end{equation}

We make the following claim which is similar to \autoref{lem:USAM Strongly Convex and Smooth}:
\begin{theorem}[USAM over Scalar Factorization Problems]\label{lem:square loss USAM}
Under the setup of \autoref{thm:square loss formal}, for any $\rho \ge 15\eta$, $\lvert x_{t+1}\rvert\ge 2\lvert x_t\rvert$ and $\lvert y_{t+1}\rvert\ge 2\lvert y_t\rvert$ for all $t\ge 1$; in other words, USAM diverges exponentially fast.
\end{theorem}
\begin{proof}
Prove by induction. From \autoref{eq:square loss USAM}, our conclusion follows once
\begin{equation*}
\eta \big \lvert (x_t+\rho x_t y_t^2) (y_t+\rho x_t^2 y_t)^2\big \rvert \ge 3 \lvert x_t\rvert,\quad \eta \big \lvert (x_t+\rho x_t y_t^2)^2 (y_t+\rho x_t^2 y_t)\big \rvert \ge 3\lvert y_t\rvert.
\end{equation*}
According to our setup that $\eta=(x_0^2+y_0^2)^{-1}$, $y_0\le 2x_0$, and the induction statement that, we have $\eta\ge (5x_t^2)^{-1}$. The second inequality then holds as long as
\begin{equation*}
\big \lvert (\rho x_t y_t^2)^2 (\rho x_t^2 y_t)\big \rvert \ge 15 x_t^2 \lvert y_t\rvert,\quad \text{i.e., } \rho^3 x_t^2 y_t^4\ge 15,
\end{equation*}
which is true as $x_t^2\ge x_0^2$, $y_t^4\ge y_0^4$, $y_0\ge x_0$, and $\rho\ge 15 \eta \ge 3 x_0^{-2}$. Note that the bounds on $\rho$ are very loose, and we made no effort to optimize it; instead, we only aimed to show that USAM starts to diverge from a $\rho=\Theta(\eta)\ll 1$. 
\end{proof}

\section{Assumptions in the Single-Neuron Linear Network Model}
\label{sec:appendix_single_neuron}

\textbf{Assumptions on $\ell$.}
Following \citet{ahn2022learning}, we make the following  assumptions about $\ell$:
\begin{enumerate}[label=\textcolor{blue}{(A\arabic*)}, ref=(A\arabic*)]
\item \label{A:cvx_lips} $\ell$ is a convex, even, $1$-Lipschitz function that is minimized at $0$.
\item \label{A:origin_rate}
$\ell$ is twice continuously differentiable near the origin with $\ell''(0) = 1$, which infers the existence of a constant $c>0$ such that $\lvert \ell'(s)\rvert\le \lvert s\rvert$ for all $\lvert s\rvert\le c$.
\item \label{A:bdd_der} We further assume a ``linear tail'' away from the minima, i.e.,  $\lvert \ell'(s)\rvert\ge c/2$ for all $\lvert s\rvert\ge c$ and $|\ell'(s)|\geq \lvert s\rvert/2$ for $|s|\leq c$.
\end{enumerate}

Some concrete example of loss functions satisfying the above assumption include a  symmetrized logistic loss $\frac{1}{2}\log(1+\exp(-2s)) + \frac{1}{2}\log(1+\exp(2s))$  and a square root loss $ \sqrt{1+s^2}$  One may refer to their paper for more details.

\section{Omitted Proof of SAM Over Single-Neuron Linear Networks}
\label{sec:appendix_sam_ell(xy)}

\begin{theorem}[Formal Version of \autoref{thm:SAM ell(xy) informal}]
\label{thm:SAM ell(xy) formal}
For a loss $\loss$ over $(x,y)\in \mathbb R^2$ defined as $\loss(x,y)=\ell(xy)$ where $\ell$ satisfies Assumptions \ref{A:cvx_lips}, \ref{A:origin_rate}, and \ref{A:bdd_der}, if the initialization $(x_0,y_0)$ satisfies:
\begin{equation*}
y_0\ge x_0\gg 0,\quad y_0^2-x_0^2=\frac \gamma \eta,\quad y_0^2=\cy \frac \gamma \eta,
\end{equation*}
where $\gamma\in [\frac 12,2]$ and $\cy\ge 1$ is constant, then for all hyper-parameter configurations $(\eta,\rho)$ such that\footnote{As $y_0\gg 0$, we must have $\eta\ll 1$; thus, these conditions automatically hold when $\eta=o(1)$ and $\rho=\mathcal O(1)$.}
\begin{equation*}
\eta \rho + \sqrt{\cy \gamma \eta}\le \min\left \{\frac 12,C\right \}\sqrt{\frac \gamma \eta},\quad \frac{4\sqrt \cy}{c}\eta^{-1}=\mathcal O(\min\{\eta^{-1.5}\rho^{-1}\gamma^{\nicefrac 12},\eta^{-2}\}),
\end{equation*}
we can decompose the trajectory of SAM (defined in \autoref{eq:definition of SAM}) into three phases, whose main conclusions are stated separately in three theorems and are informally summarized here:
\begin{enumerate}
\item \textbf{(\autoref{lem:phase1 SAM})} Until $x_t=\mathcal O(\sqrt{\gamma \eta})$, we must have $y_t=\Omega(\sqrt{\gamma/\eta})$, and $x_{t+1}\le x_t-\Omega(\sqrt{\gamma\eta})$.
\item 
\textbf{(\autoref{lem:phase2 SAM})} After \pone and until $y_t\le \lvert x_t\rvert$, $\lvert x_t\rvert=\mathcal O(\eta \rho+\sqrt{\eta})$ still holds. Meanwhile, $y_{t+1}\le y_t-\min\{\Omega(\eta \rho^2) y_t,\Omega(\eta \rho)\}$ (i.e., $y_t$ either drops by $\Omega(\eta \rho)$ or decays by $\Omega(\eta \rho^2)$).
\item \textbf{(\autoref{lem:phase3 SAM})} After \ptwo, we always have $\lvert x_t\rvert,\lvert y_t\rvert=\mathcal O(\eta \rho+\sqrt{\eta})$.
\end{enumerate}
\end{theorem}

\subsection{Basic Properties and Notations}
Recall the update rule of SAM in \autoref{eq:definition of SAM}: $w_{t+1}\gets w_t-\eta \nabla \loss(w_t+\rho \frac{\nabla \loss(w_t)}{\lVert \nabla \loss(w_t)\rVert})$.
By writing $w_t$ as $\begin{bmatrix}x_t\\y_t\end{bmatrix}$, substituting $\loss(x,y)=\ell(xy)$, and utilizing the expressions of $\nabla \loss(x,y)$ in \autoref{eq:hessian_lxy}, we have:
\begin{equation}
w_t+\rho \frac{\nabla \loss(w_t)}{\lVert \nabla \loss(w_t)\rVert}
=\begin{bmatrix}x_t\\y_t\end{bmatrix}+\rho\frac{\ell'(x_ty_t)}{\lvert \ell'(x_ty_t)\rvert \sqrt{x_t^2+y_t^2}}\begin{bmatrix}y_t\\x_t\end{bmatrix}
=\begin{bmatrix}x_t\\y_t\end{bmatrix}+\rho\frac{\sgn(x_ty_t)}{\sqrt{x_t^2+y_t^2}}\begin{bmatrix}y_t\\x_t\end{bmatrix},\label{eq:SAM in ell(xy)}
\end{equation}
where the second step uses the fact that $\ell$ is a even function so $\ell'(t)$ has the same sign with $t$.
Define $z_t=\frac{\sqrt{x_t^2+y_t^2}}{\sgn(x_ty_t)}$.
Then $w_t+\rho \frac{\nabla \loss(w_t)}{\lVert \nabla \loss(w_t)\rVert}=\begin{bmatrix}x_t+\rho y_t z_t^{-1}\\y_t+\rho x_tz_t^{-1}\end{bmatrix}$.
Further denoting $\ell'\big ( (x_t+\rho y_tz_t^{-1})(y_t +\rho x_t z_t^{-1})\big )$ by $\ell_t'$, one can simplify the update rule \autoref{eq:definition of SAM} as follows:
\begin{align}
w_{t+1}&=\begin{bmatrix}x_t\\y_t\end{bmatrix}-\eta \ell'\big ( (x_t+\rho y_tz_t^{-1})(y_t +\rho x_t z_t^{-1})\big )  \begin{bmatrix}y_t+\rho x_t z_t^{-1}\\x_t+\rho y_tz_t^{-1}\end{bmatrix} \nonumber\\
&=\begin{bmatrix}x_t\\y_t\end{bmatrix}-\eta \ell'_t \begin{bmatrix}y_t+\rho x_t z_t^{-1}\\x_t+\rho y_tz_t^{-1}\end{bmatrix}.\label{eq:SAM without approximation}
\end{align}

\textbf{Assumption.}
Following \autoref{remark:ill-definedness}, we assume that $x_t,y_t\neq 0$ for all $t$.

We are ready to give some basic properties of \autoref{eq:SAM without approximation}. First, we claim that the sign of $\ell_t'$ is the same as the product $x_ty_t$. Formally, we have the following lemma:
\begin{lemma}[Sign of Gradient in SAM] \label{lem:sign SAM}
If $x_t\neq 0,y_t\neq 0$, then  $\sign(x_t) = \sign(x_t+\rho y_tz_t^{-1})$ and $\sign(y_t)=\sign(y_t +\rho x_t z_t^{-1})$.
In particular, if $x_t\neq 0,y_t\neq 0$, we must have  $\sign(\ell'_t) = \sign(x_t y_t)$.
\end{lemma}
\begin{proof}
Note that $\sign(z_t) = \sign(x_ty_t)$, so  $\sign(y_t z_t^{-1}) = \sign (x_t)$.
Similarly, $\sign(x_tz_t^{-1}) = \sign(y_t)$.
Thus, we have $\sign(x_t) = \sign(x_t+\rho y_tz_t^{-1})$ and $\sign(y_t)=\sign(y_t +\rho x_t z_t^{-1})$,
which in turn shows that $\sign(\ell'_t) = \sign\left(\sign(x_t+\rho y_tz_t^{-1})\sign(y_t +\rho x_t z_t^{-1})\right) =\sign(x_ty_t)$.
\end{proof}

Using \autoref{lem:sign SAM}, the SAM update can be equivalently written as
\begin{align}
\begin{bmatrix}x_{t+1}\\y_{t+1}\end{bmatrix}
&=\begin{bmatrix}x_t\\y_t\end{bmatrix}-\eta \ell'_t \begin{bmatrix}y_t\\x_t\end{bmatrix}-\eta |\ell'_t| \begin{bmatrix} \rho x_t (x_t^2+y_t^2)^{-\nicefrac 12}\\ \rho y_t (x_t^2+y_t^2)^{-\nicefrac 12}\end{bmatrix}.\label{eq:SAM simplified}
\end{align}

We then claim the following property, which can be intuitively interpolated as SAM always goes towards the minimizes (i.e., both axes) for each step, although sometimes it may overshoot.
\begin{lemma}  \label{lem:monotone SAM}
Suppose that  $x_t,y_t\ne 0$.
Then, the following basic properties hold: 
\begin{itemize}
\item If $x_t>0$, then $x_{t+1}<x_t$. If $x_t<0$, then $x_{t+1}>x_t$.
\item If $y_t>0$, then $y_{t+1}<y_t$. If $y_t<0$, then $y_{t+1}>y_t$.
\end{itemize} 
\end{lemma}
\begin{proof}
We only show the case that $x_t>0$ as the other cases are similar. Recall the dynamics of $x_t$:
\begin{align*}
x_{t+1} &= x_t -\eta \ell_t' \cdot \rho z_t^{-1}x_t - \eta \ell_t' \cdot y_t\,.  
\end{align*} 
Since $x_t\ne 0$, \autoref{lem:sign SAM} implies that  $\sign(\ell'_t) = \sign(x_ty_t)$, and so 
\begin{align*}
x_{t+1} &= x_t -\eta \ell_t' \cdot \rho z_t^{-1}x_t - \eta \ell_t' \cdot y_t \\
&= x_t -\eta \sgn(x_ty_t)\lvert \ell_t'\rvert \cdot \rho \sgn(x_ty_t)(x_t^2+y_t^2)^{-\nicefrac 12}x_t - \eta \sgn(x_ty_t)\lvert \ell_t'\rvert \cdot \sgn(y_t)\lvert y_t\rvert\\
&= x_t -\eta |\ell_t'| \cdot \rho (x_t^2+y_t^2)^{-\nicefrac 12} x_t - \eta |\ell_t'| \cdot |y_t| \leq x_t,
\end{align*} 
where the last line uses the assumption that $x_t>0$.
\end{proof}

\subsection{\pone: $x_t$ Decreases Fast while $y_t$ Remains Large}
As advertised in \autoref{thm:SAM ell(xy) formal}, we group all rounds until $x_t\le \frac c4\sqrt{\gamma \eta}$ as \pone. Formally:
\begin{definition}[\pone]\label{def:definition of Phase I}
Let $\tone$ be the largest time  such that $x_t>\frac c4\sqrt{\gamma \eta}$ for all $t\le \tone$. 
We call the iterations $[0, \tone ]$ \textbf{\pone}.
\end{definition}

\begin{theorem}[Main Conclusion of \pone; SAM Case]\label{lem:phase1 SAM}
For $\tzero =\Theta(\min\{\eta^{-1.5}\rho^{-1}\gamma^{\nicefrac 12},\eta^{-2}\})$,
we have $y_t
\geq \frac 12\sqrt{\gamma/\eta}$ for all $t\le \min\{\tzero,\tone\}$ under the conditions of \autoref{thm:SAM ell(xy) formal}.
Moreover, we have $x_{t+1}\le x_t-\frac c4\sqrt{\gamma \eta}$ for all $t\le \min\{\tzero,\tone\}$, which consequently infers $\tone\le \tzero$ under the conditions of \autoref{thm:SAM ell(xy) formal}, i.e., $\min\{\tzero,\tone\}$ is just $\tone$. This shows the first claim of \autoref{thm:SAM ell(xy) formal}.
\end{theorem}
\begin{remark*}
This theorem can be intuitively explained as follows:
At initialization, $x_0,y_0$ are both $\Theta(\sqrt{1/\eta})$. However, after $\tone$ iterations, we have $|x_{\tone+1}| = \mathcal O(\sqrt{\eta})$;
meanwhile, $y_t$ is still of order $\Omega(\sqrt{1/\eta})$ (much larger than $\mathcal O(\sqrt{\eta})$). Hence, $\lvert x_t\rvert$ and $\lvert y_t\rvert$ get widely separated in \pone.
\end{remark*}
\begin{proof}
This theorem is a direct consequence of \autoref{lem:y_t cannot be too small in I SAM} and \autoref{lem:phase1}.
\end{proof}

\begin{lemma}\label{lem:y_t cannot be too small in I SAM} 
There exists $\tzero =\Theta(\min\{\eta^{-1.5}\rho^{-1}\gamma^{\nicefrac 12},\eta^{-2}\})$ such that $|y_t|
\geq \frac 12\sqrt{\gamma/\eta}$ for $t\le  \tzero $. 
If $\eta$ is sufficiently small s.t.  $ \eta \rho +  \sqrt{\cy  \gamma\eta } \leq \frac 12 \sqrt{  \gamma / \eta}$, 
then $y_t \geq \frac 12 \sqrt{  \gamma / \eta}$ for $t\leq \min\{\tone,\tzero\}$.
\end{lemma} 
\begin{proof}
Let $\tzero$ be defined as follows:
\begin{align}
\tzero \triangleq \max\left\{t~:~ \left (1- \frac{8}{\sqrt{\cy}}  \eta^{1.5} \gamma^{-\nicefrac 12} \rho    - \eta^2 \right )^t \geq \frac{1}{4}\right\}\,.
\end{align}
Then, it follows that $\tzero =\Theta(\min\{\eta^{-1.5}\rho^{-1}\gamma^{\nicefrac 12},\eta^{-2}\})$.
We first prove the following claim.
\begin{claim}\label{claim:square of y_t LB}
For all $t\leq \tzero$, $y_t^2 \geq \frac{1}{4}\gamma/\eta$.
\end{claim}
\begin{proof}
We prove this claim by induction.
We trivially have $y_0^2 = {\cy\gamma}/{ \eta}\ge \frac 14 \cy\gamma /\eta$. Now suppose that the conclusion holds up to some $t<\tzero$, i.e., $y_{t'}^2 \geq \frac{1}{4}\gamma/\eta$ for all $t'\le t$. We consider $y_{t+1}^2-x_{t+1}^2$:
\begin{align*}
y_{t+1}^2 - x_{t+1}^2 &=
(1-\eta \ell'_t \cdot \rho z_t^{-1})^2 (y_t^2-x_t^2) - (\eta \ell'_t )^2 (y_t^2 -x_t^2 ) \\
&= (1-2\eta \ell'_t \cdot \rho z_t^{-1} + \eta^2 (\ell'_t \cdot \rho z_t^{-1})^2  - \eta^2 (\ell'_t)^2  )\cdot (y_t^2 -x_t^2)\\
&\overset{(a)}{\geq} (1-2\eta \rho z_t^{-1}   - \eta^2 )\cdot (y_t^2 -x_t^2)\\
&\geq (1-2  \eta \rho y_t^{-1}   - \eta^2 )\cdot (y_t^2 -x_t^2)\\
&\overset{(b)}{\geq} (1- \frac{8}{\sqrt{\cy}}  \eta^{1.5} \gamma^{-\nicefrac 12} \rho    - \eta^2 )\cdot (y_t^2 -x_t^2),
\end{align*}
where (a) used $\ell'_t\le 1$ (thanks to Assumption \ref{A:cvx_lips}) and (b) used the induction hypothesis $y_t^2\ge  \frac 14 \cy\gamma /\eta$. 
This further implies that
\begin{equation*}
y_{t+1}^2-x_{t+1}^2 \ge \left (1- \frac{8}{\sqrt{\cy}}  \eta^{1.5} \gamma^{-\nicefrac 12} \rho    - \eta^2 \right )^{t+1}(y_0^2-x_0^2) = \left (1- \frac{8}{\sqrt{\cy}}  \eta^{1.5} \gamma^{-\nicefrac 12} \rho    - \eta^2 \right )^{t+1} \frac{\gamma}{\eta }.
\end{equation*}
Thus, by the definition of $\tzero$, we must have $y_{t+1}^2\ge y_{t+1}^2-x_{t+1}^2\ge \frac 14  \gamma / \eta$, which proves the claim.
\end{proof}

Next, we prove the second conclusion.
\begin{claim}\label{claim:raw y_t LB}
If $ \eta \rho +  \sqrt{ \cy \gamma\eta } \leq \frac 12 \sqrt{ \gamma / \eta}$, then $y_t \geq \frac 12 \sqrt{  \gamma / \eta}$ for $t\leq \min\{\tzero,\tone\}$.
\end{claim}
\begin{proof}
Still show by induction. Let $y_t\geq \frac 12 \sqrt{ \gamma / \eta}>0$ for some $t<\min\{\tzero,\tone\}$. Consider $y_{t+1}$.

By \autoref{def:definition of Phase I}, $x_t$ is positive for all $t\le \tone$. Thus, using \autoref{lem:monotone SAM}, we have $x_t\leq x_0\leq y_0 = \sqrt{\cy \gamma/\eta}$. 
Since $x_t, y_t>0$, we have $\sign(\ell'_t)=\sign(x_ty_t)=\sign(x_t)>0$ and hence \eqref{eq:SAM simplified} gives
\begin{align*}
y_{t+1}  &= y_t- \eta |\ell'_t|  \rho  (x_t^2+y_t^2)^{-\nicefrac 12} |y_t|-\eta  |\ell'_t| \lvert x_t\rvert\\
&\geq y_t -\eta \rho - \eta |x_t|,
\end{align*}
where we used the facts that $\lvert \ell_t'\rvert\le 1$ and $(x_t^2+y_t^2)^{-\nicefrac 12}\le \min\{\lvert x_t\rvert^{-1},\lvert y_t\rvert^{-1}\}$. Hence, we get
\begin{align*}
y_{t+1}  \geq y_t -\eta \rho -  \sqrt{\cy  \gamma\eta } \geq 0 
\end{align*}  
since $ \eta \rho +  \sqrt{\cy  \gamma\eta } \leq \frac 12 \sqrt{  \gamma / \eta}\le y_t$. 
By \autoref{claim:square of y_t LB}, $y_{t+1}\geq 0$ implies $y_{t+1}\geq \frac 12 \sqrt{ \gamma / \eta}$ as well.
\end{proof}

Combining \autoref{claim:square of y_t LB} and \autoref{claim:raw y_t LB} finishes the proof of \autoref{lem:y_t cannot be too small in I SAM}.
\end{proof}

\begin{lemma} \label{lem:phase1}
For any $t\leq \min\{\tzero,\tone\}$,  we have 
\begin{align*}
x_{t+1} \le x_t-\frac c4\sqrt{\gamma\eta}\,.
\end{align*} 
In particular, if $\eta$ is sufficiently small s.t. $\frac{4\sqrt \cy}{c} \eta^{-1}\leq \tzero$, then we must have $\tone \leq \frac{4}{c} \eta^{-1}-1<\tzero$. 
\end{lemma}
 
\begin{proof}
Since $x_t>0$ for all $t\leq \tone$, we have $\sign(\ell'_t) = \sign(x_ty_t) = \sign(y_t)$ from \autoref{lem:sign SAM} and \autoref{lem:y_t cannot be too small in I SAM}, so the SAM update \eqref{eq:SAM simplified} becomes
\begin{equation*}
x_{t+1} = x_t -\eta |\ell_t'| \cdot \rho (x_t^2+y_t^2)^{-\nicefrac 12}  x_t - \eta |\ell_t'| \cdot |y_t|.  
\end{equation*} 
Since $x_t> 2c \sqrt{\frac{ \eta}{ \cy\gamma}}$, we have
\begin{align*}
(x_t+\rho y_tz_t^{-1})(y_t +\rho x_t z_t^{-1}) > x_ty_t \geq   2 c \sqrt{\frac{ \eta}{  \gamma}} \cdot \frac 12 \sqrt{\frac{  \gamma}{\eta}} \geq c \,,
\end{align*}
which implies that $\ell'_t \geq c/2$ from Assumption \ref{A:bdd_der}.
Together with \autoref{lem:y_t cannot be too small in I SAM}, we have  
\begin{align*}
x_{t+1}&= x_t -\eta |\ell_t'| \cdot \rho (x_t^2+y_t^2)^{-\nicefrac 12}   z_t^{-1}x_t - \eta |\ell_t'| \cdot |y_t|  \\
&\leq  x_t -    \eta |\ell_t'| \cdot |y_t| 
\le x_t-\eta \frac c2 \frac 12\sqrt{\frac {\gamma}{\eta}}=x_t-\frac c4\sqrt{\gamma\eta}. 
\end{align*}
Let $\tone'\triangleq \frac{\sqrt{\cy \gamma/\eta}}{\frac c4\sqrt{\gamma\eta}} = \frac{4\sqrt \cy}{c} \eta^{-1}$. 
Since $x_0\le \sqrt{\cy \gamma/\eta}$, we have $x_{\tone'}<  \frac c4\sqrt{\gamma\eta}$ as long as $\tone\leq \tzero$.
Thus, it follows that $\tone \leq \tone'-1<\tzero$, as desired.
\end{proof}

\subsection{\ptwo: $y$ Keeps Decreasing Until Smaller Than $\lvert x_t\rvert$}
Then we move on to the second claim of \autoref{thm:SAM ell(xy) formal}. We define all rounds until $y_t<\lvert x_t\rvert$ that are after \pone as \ptwo. Formally, we have the following definition.
\begin{definition}[\ptwo]\label{def:definition of Phase II}
Let $\ttwo$ be the first time that $y_t<|x_t|$. We call $(\tone,\ttwo]$ \textbf{\ptwo}. 
\end{definition}

Before presenting the main conclusion of \ptwo, we first define a threshold $B$ that we measure whether a variable is ``small enough'', which we will use throughout this section. Formally,
\begin{equation*}
B\triangleq \max\left\{ \frac 2 c \sqrt{\frac{ \eta}{ \cy\gamma}},~\eta \rho +\sqrt{\cy \eta \gamma}\right\}.
\end{equation*}

\begin{theorem}[Main Conclusion of \ptwo; SAM Case]\label{lem:phase2 SAM}
Under the conditions of \autoref{thm:SAM ell(xy) formal}, $\lvert y_t\rvert\le \sqrt{\cy \gamma / \eta}$ and $\lvert x_t\rvert\le B$ throughout \ptwo.
Under the same conditions, we further have $y_{t+1}\le y_t-\min\{\frac 12\eta \rho^2 y_t,\frac{c}{2\sqrt 2}\eta \rho\}$ for all $\tone\le t\le \ttwo$ -- showing the second claim of \autoref{thm:SAM ell(xy) formal}.
\end{theorem}
\begin{remark*}
This theorem can be understood as follows. Upon entering \ptwo, $\lvert x_t\rvert$ is bounded by $\mathcal O(\eta \rho + \sqrt \eta)$. This then gets preserved throughout \ptwo. Meanwhile, in $\ttwo$ iterations, $y_t$ drops rapidly such that $y_{\ttwo+1}\le \lvert x_{\ttwo+1}\rvert=\mathcal O(\eta \rho + \sqrt \eta)$. In other words, $x_t$ and $y_t$ are both ``close enough'' to 0 after \ptwo; thus, SAM finds the flattest minimum (which is the origin).
\end{remark*}
\begin{proof}
The claims are shown in \autoref{lem:bounded y_t in Phase II}, \autoref{lem:bounded x_t in Phase II} and \autoref{lem:phase2}, respectively.
\end{proof}

\begin{lemma}\label{lem:bounded y_t in Phase II}
If $\eta \rho+    \sqrt{ {\cy \gamma\eta }} < \sqrt{\cy\gamma/\eta}$, then during the \ptwo,   $|y_t|\leq y_0\leq \sqrt{\cy\gamma/\eta}$. 
\end{lemma}
\begin{proof}
When $\sign(y_{t+1}) = \sign(y_t)$, then \autoref{lem:monotone SAM} implies that $|y_{t+1}|\leq |y_t|$.
Now consider the case $\sign(y_{t+1}) = -\sign(y_t)$. 
For simplicity, say $y_t >0$ and $ y_{t+1}<0$.
Since $y_t>0$, we have $\sign(\ell'_t)=\sign(x_ty_t)=\sign(x_t)$ and hence \autoref{eq:SAM simplified} gives
\begin{align*}
y_{t+1} &= y_t- \eta |\ell'_t|  \rho  (x_t^2+y_t^2)^{-\nicefrac 12} |y_t|-\eta  |\ell'_t| \lvert x_t\rvert\\
&\geq -\eta \rho - \eta |x_t| \geq -\eta \rho - \eta y_t\\
&\geq -\eta \rho - \eta \sqrt{\frac{\cy \gamma}{\eta}} =   -\eta \rho -   \sqrt{ {\cy \gamma\eta }}\,,
\end{align*}
which shows that $|y_{t+1}|\leq \eta \rho+    \sqrt{ {\cy \gamma\eta }}$. This proves the statement.
\end{proof}

\begin{lemma}\label{lem:bounded x_t in Phase II}
Suppose that $|x_t|\leq B=\max\left\{ \frac 2 c \sqrt{\frac{ \eta}{ \cy\gamma}},~\eta \rho +\sqrt{\cy \eta \gamma}\right\}$.
Then  we have 
$|x_{t+1}| \leq  B$.
\end{lemma}
\begin{proof}
By symmetry, we only consider the case where $x_{t} > 0$.  
If $x_{t+1}\geq 0$,  then we have $0\leq x_{t+1} \leq x_t$ due to \autoref{lem:phase1}.
If $x_{t+1}<0$, then since $x_{t} >0$, it follows  that 
\begin{align}
x_{t+1} &= \sgn(x_{t})x_{t+1} = |x_t| -\eta |\ell'_t|   \rho  (x_t^2+y_t^2)^{-\nicefrac 12} |x_{t}|   - \eta |\ell'_t|  |y_{t}|\nonumber\\
&\ge |x_t| - \eta \rho -\eta\sqrt{\cy\gamma/\eta} \geq - \eta \rho -\sqrt{\cy \eta \gamma}\,,\label{eq:x_t bound in II SAM}
\end{align}
where the last inequality is because $|y_{t}|\le y_0\leq \sqrt{\cy\gamma/\eta}$. This concludes the proof.
\end{proof}
\begin{corollary}
Note that, by definition of \pone in \autoref{def:definition of Phase I}, we already have $\lvert x_{\tone+1}\rvert\le B$. Hence, this lemma essentially says $\lvert x_t\rvert\le B=\max\left\{ \frac 2 c \sqrt{\frac{ \eta}{ \cy\gamma}},~\eta \rho +\sqrt{\cy \eta \gamma}\right\}$ for all $\tone<t\le \ttwo$.
\end{corollary}

\begin{lemma}\label{lem:phase2}
If $y_t \geq |x_t|$, then we must have
\begin{equation*}
y_{t+1} \leq y_t- \min\left\{   \frac{1}{2}\eta \rho^2 y_t,~~\frac{c}{2\sqrt{2}}\eta \rho  \right\} \,. 
\end{equation*}
\end{lemma}
\begin{proof}
Since $y_t \geq |x_t|$, we have $(x_t^2+y_t^2)^{-\nicefrac 12}\ge (2y_t^2)^{-\nicefrac 12}=(\sqrt 2 y_t)^{-1}$. Consider two cases:
\begin{enumerate}
\item [i)] if $\lvert (x_t+\rho y_t z_t^{-1}) (y_t+\rho x_t z_t^{-1})\rvert\ge c$, then $\lvert \ell_t'\rvert\in [\frac c2,1]$ according to Assumption \ref{A:bdd_der}. Moreover, since $\sign(\ell'_t) = \sign(x_ty_t)=\sign(x_t)$, it follows that
\begin{align}
y_{t+1}&=   y_t  -\eta |\ell_t'| \cdot \rho (x_t^2+y_t^2)^{-\nicefrac 12}y_t - \eta \ell_t' \cdot x_t\nonumber\\
&=\left (1-\eta |\ell_t'|   \rho  (x_t^2+y_t^2)^{-\nicefrac 12}\right )y_t-\eta |\ell'_t|   |x_t|\nonumber\\
&\le \left (1-\eta c_t \rho  (\sqrt 2 y_t)^{-1}\right )y_t\nonumber\\
&\le y_t-\eta \rho \frac c2 \frac{1}{\sqrt 2}=y_t-\Omega(\eta \rho).\label{eq:Phase II Case 1}
\end{align}
\item [ii)] otherwise, it follows from \ref{A:bdd_der} that 
\begin{align*}
|\ell'_t| &\geq \frac{1}{2} |(x_t+\rho y_t z_t^{-1}) (y_t+\rho x_t z_t^{-1})|\,.
\end{align*} 

As $\sgn(x_t)=\sgn(\rho y_tz_t^{-1})$ and $\sgn(y_t)=\sgn(\rho x_t z_t^{-1})$, it follows that 
\begin{align*}
|\ell_t'| \cdot \rho z_t^{-1}y_t  &\geq \frac{1}{2} |(x_t+\rho y_t z_t^{-1}) (y_t+\rho x_t z_t^{-1})|\cdot \rho z_t^{-1}y_t\\
&\geq  \frac{1}{2} (\rho y_t z_t^{-1})^2 y_t 
\end{align*}

we must have
\begin{equation}
y_{t+1}\le y_t-\frac{1}{2}\eta(\rho y_t z_t^{-1})^2 (y_t)\leq y_t - \frac{1}{2}\eta \rho^2 y_t = (1-\Omega(\eta \rho^2))y_t,\label{eq:Phase II Case 2}
\end{equation}
where we used $z_t^{-1}\ge (\sqrt 2 y_t)^{-1}$ and thus $y_t z_t^{-1}\geq 1/\sqrt{2}$.
\end{enumerate}

Combining \autoref{eq:Phase II Case 1} and \autoref{eq:Phase II Case 2}, we obtain the desired conclusion. 
\end{proof}

\subsection{\pthree: Both $x_t$ and $y_t$ Oscillates Around the Origin}
It only remains to show that the iteration never escapes the origin.
\begin{definition}[\pthree]
We denote by ``\textbf{\pthree}'' all iterations after $
\ttwo$.
\end{definition}
\begin{theorem}[Main Conclusion of \pthree; SAM Case]
\label{lem:phase3 SAM}
For all $t>\ttwo$, we have $\lvert x_t\rvert,\lvert y_t\rvert\le B$ where $B=\max\left\{\frac 2 c \sqrt{\frac{ \eta}{ \cy\gamma}},~\eta \rho +\sqrt{\cy \eta \gamma}\right\}$.
\end{theorem}
\begin{remark*}
As we will see shortly, when entering \pthree, both $\lvert x_t\rvert$ and $\lvert y_t\rvert$ are bounded by $\mathcal O(\eta \rho + \sqrt \eta)$. This theorem essentially says that they can never exceed this bound in \pthree. In other words, in \pthree, both $x_t$ and $y_t$ are oscillating around $0$.
\end{remark*}
\begin{proof}
We first show the following lemma, ruling out the possibility of $\lvert y_t\rvert>B$ after \ptwo:
\begin{lemma}\label{lem:bounded y_t after Phase II}
If $\eta \rho + \eta B\le B$,\footnote{Recall that $B=\max\{\frac 2 c \sqrt{\frac{ \eta}{ \cy\gamma}},~\eta \rho +\sqrt{\cy \eta \gamma}\}$, we only need to ensure that $\eta B\le \sqrt{\cy \eta \gamma}$, which can be done by $\eta\le \min\{cC\gamma,\frac 14 C \gamma \rho^{-2},\frac 12\}$. As the RHS is of order $\Omega(1)$, $\eta=\mathcal O(1)$ again suffices.}
then $\lvert y_{\ttwo+1}\rvert\le B=\max\left\{\frac 2 c \sqrt{\frac{ \eta}{ \cy\gamma}},~\eta \rho +\sqrt{\cy \eta \gamma}\right\}$.
\end{lemma}
\begin{proof}
The proof will be very similar to \autoref{lem:bounded x_t in Phase II}. Recall that by \autoref{def:definition of Phase II}, we must have $y_{\ttwo}>0$. If $y_{\ttwo+1}>0$ as well, then $\lvert y_{\ttwo+1}\rvert=y_{\ttwo+1}\le \lvert x_{\ttwo+1}\rvert\le B$. Otherwise,
\begin{equation*}
y_{\ttwo+1} = \sgn(y_{\ttwo})y_{\ttwo+1} = |y_{\ttwo}| -\eta |\ell'_{\ttwo}|   \rho  (x_{\ttwo}^2+y_{\ttwo}^2)^{-\nicefrac 12} |y_{\ttwo}|   - \eta |\ell'_{\ttwo}|  |x_{\ttwo}|\ge -\eta \rho -\eta B,
\end{equation*}
where we used $\lvert x_{\ttwo}\rvert\le B$. We proved our claim as $\eta \rho + \eta B\le B$.
\end{proof}

According to \autoref{lem:bounded x_t in Phase II} \autoref{lem:bounded y_t after Phase II}, we have
\begin{equation}\label{eq:bound on x and y when exiting II}
\lvert x_{\ttwo+1}\rvert,\lvert y_{\ttwo+1}\rvert\le B=\max\left\{\frac 2 c \sqrt{\frac{ \eta}{ \cy\gamma}},~\eta \rho +\sqrt{\cy \eta \gamma}\right\}.
\end{equation}
Hence, it only remains to do a reduction, stated as follows.
\begin{lemma}\label{lem:Phase IV SAM}
If $\eta \rho+\eta B\le B$, then $\lvert x_{t+1}\rvert,\lvert y_{t+1}\rvert\le B$ as long as $\lvert x_{t}\rvert,\lvert y_{t}\rvert\le B$.
\end{lemma}

\begin{proof}
Without loss of generality, let $x_t>0$ for some $t>\ttwo$ and show that $\lvert x_{t+1}\rvert\le B$ -- which is identical to the proof of \autoref{lem:bounded y_t in Phase II}. By symmetry, the same also holds for $\lvert x_t\rvert$.
\end{proof}

Our conclusion thus follows from an induction based on \autoref{eq:bound on x and y when exiting II} and \autoref{lem:Phase IV SAM}.
\end{proof}  

\section{Omitted Proof of USAM Over Single-Neuron Linear Networks}
\label{sec:appendix_usam_ell(xy)}

\begin{theorem}[Formal Version of \autoref{thm:USAM ell(xy) informal}]\label{thm:USAM ell(xy) formal}
For a loss $\loss$ over $(x,y)\in \mathbb R^2$ defined as $\loss(x,y)=\ell(xy)$ where $\ell$ satisfies Assumptions \ref{A:cvx_lips}, \ref{A:origin_rate}, and \ref{A:bdd_der}, if the initialization $(x_0,y_0)$ satisfies:
\begin{equation*}
y_0\ge x_0\gg 0,\quad y_0^2-x_0^2=\frac \gamma \eta,\quad y_0^2=\cy \frac \gamma \eta,
\end{equation*}
where $\gamma\in [\frac 12,2]$ and $\cy\ge 1$ is constant, then for all hyper-parameter configurations $(\eta,\rho)$ 
such that\footnote{Similar to \autoref{thm:SAM ell(xy) formal}, these conditions are satisfied once $\eta=o(1)$ and $\rho=\mathcal O(1)$.}
\begin{equation*}
\eta\le \frac 12,\quad \eta\rho \le \min\left \{\frac 12,\cy^{-1}\gamma^{-1}\right \},\quad 2\sqrt \cy \eta^{-1}=\mathcal O(\min\{\eta^{-1}\rho^{-1},\eta^{-2}\}),\quad \cy \gamma \left (1+\rho \cy \frac \gamma \eta\right )\ge 16,
\end{equation*}
we can characterize the initial and final phases of the trajectory of USAM (defined in \autoref{eq:definition of USAM}) by the following two theorems, whose main conclusions are informally summarized below:
\begin{enumerate}
\item \textbf{(\autoref{lem:phase1 USAM})} Until $x_t=\mathcal O(\sqrt{\gamma \eta})$, we must have $y_t=\Omega(\sqrt{\gamma/\eta})$, and $x_{t+1}\le x_t-\Omega(\sqrt{\gamma\eta})$.
\item \textbf{(\autoref{lem:phase3 usam})} Once $(1+\rho y_t^2) y_t^2\lesssim \nicefrac 2\eta$, $\lvert x_t\rvert$ decays exponentially and thus USAM gets stuck.
\end{enumerate}
\end{theorem}

Different from \autoref{thm:SAM ell(xy) formal}, there is no characterization of \ptwo here, which means the iterates can also stop above the threshold $\tilde y_{\text{USAM}}^2$ defined in \autoref{eq:threshold in USAM} (the technical reason is sketched in \autoref{sec:technical ell(xy)}, i.e., SAM gradients non non-negligible when $y_t$ is large, while USAM gradients vanish once $\lvert x_t\rvert$ is small). 
However, we remark that the main takeaway of \autoref{thm:USAM ell(xy) formal} is to contrast SAM (which always attains $y_\infty^2\approx 0$) with USAM (which must stop once $(1+\rho y_t^2)y_t^2\lesssim \nicefrac 2\eta$).

\subsection{Basic Properties and Notations}
Recall the update rule of USAM in \autoref{eq:definition of USAM}: $w_{t+1}\gets w_t-\eta \nabla \loss(w_t+\rho \nabla \loss(w_t))$.
Still writing $w_t$ as $[x_t\quad y_t]^\top$, we have
\begin{align}
w_{t+1}
&=\begin{bmatrix}x_t\\y_t\end{bmatrix}-\eta \nabla \loss\left (\begin{bmatrix}x_t\\y_t\end{bmatrix}+\rho \ell'(x_ty_t)\begin{bmatrix}y_t\\x_t\end{bmatrix}\right )\nonumber\\
&=\begin{bmatrix}x_t\\y_t\end{bmatrix}-\eta \ell'\big ((x_t+\rho \ell'(x_ty_t)y_t)(y_t+\rho \ell'(x_ty_t)x_t)\big )\begin{bmatrix}y_t+\rho \ell'(x_ty_t)x_t\\x_t +\rho \ell'(x_ty_t)y_t\end{bmatrix}.\label{eq:USAM without approximation}
\end{align}

Due to the removal of normalization, $\ell'$ are taken twice at different points in \autoref{eq:USAM without approximation}. For simplicity, we denote $\tilde \ell'_t=\ell'(x_ty_t)$ and $\ell'_t=\ell'((x_t+\rho \tilde \ell_t'y_t)(y_t+\rho \tilde \ell_t'x_t))$. The update rule can be rewritten as:
\begin{equation}\label{eq:USAM simplified}
x_{t+1}=x_t-\eta \ell_t'\cdot (y_t+\rho \tilde \ell_t' x_t),\quad y_{t+1}=y_t-\eta \ell_t'\cdot (x_t+\rho \tilde \ell_t' y_t).
\end{equation}

Similar to the SAM case, we shall approximate $\tilde \ell_t'$ and $\ell_t'$ according to the magnitude of $x_ty_t$ and $(x_t+\rho \tilde \ell_t' y_t)(y_t+\rho \tilde \ell_t' x_t)$. We first have the following lemma similar to \autoref{lem:sign SAM}:
\begin{lemma}[Sign of Gradient in USAM] \label{lem:sign USAM}
If $x_t\neq 0,y_t\neq 0$, then  $\sign(x_t) = \sign(x_t+\rho \tilde \ell_t' y_t)$ and $\sign(y_t)=\sign(y_t+\rho \tilde \ell_t' x_t)$.
In particular, if $x_t\neq 0,y_t\neq 0$, then  $\sign(\ell'_t) = \sgn(\tilde \ell_t') = \sign(x_t y_t)$.
\end{lemma}
\begin{proof}
First of all, $\sgn(\tilde \ell_t') = \sign(x_ty_t)$ according to Assumption \ref{A:cvx_lips}.
Therefore,
\begin{align*}
\sgn(x_t+\rho \tilde \ell_t' y_t)&=\sgn(x_t+\rho \sgn(x_ty_t)y_t)=\sgn(x_t),\\
\sgn(y_t+\rho \tilde \ell_t' x_t)&=\sgn(y_t+\rho \sgn(x_ty_t)x_t)=\sgn(y_t),
\end{align*}
giving our first two claims. The last conclusion follows by definition.
\end{proof}

Moreover, we also have the following lemma analog to \autoref{lem:monotone SAM}:
\begin{lemma}  \label{lem:monotone USAM}
Suppose that  $x_t,y_t\geq 0$.
Then, the following basic properties hold: 
\begin{itemize}
\item If $x_t>0$, then $x_{t+1}<x_t$. If $x_t<0$, then $x_{t+1}>x_t$.
\item If $y_t>0$, then $y_{t+1}<y_t$. If $y_t<0$, then $y_{t+1}>y_t$.
\end{itemize} 
\end{lemma}
\begin{proof}
We only prove the statement for $x_t>0$ as the proof is similar for other cases. Recall the dynamics of $x_t$:
\begin{align*}
x_{t+1} &= x_t-\eta \ell_t'\cdot (y_t+\rho \tilde \ell_t' x_t).  
\end{align*} 
Since $x_t\ne 0$, \autoref{lem:sign USAM} implies that  $\sign(\ell'_t) = \sign(\tilde \ell'_t) = \sign(x_ty_t)$, and so 
\begin{align*}
x_{t+1} &= x_t-\eta \ell_t'\cdot (y_t+\rho \tilde \ell_t' x_t)\\
&= x_t - \eta \sgn(x_ty_t) \lvert \ell_t'\rvert \cdot (y_t+\rho \sgn(x_ty_t) \lvert \tilde \ell_t'\rvert x_t)\\
&=x_t - \eta \lvert \ell_t'\rvert \cdot (\lvert y_t\rvert+\rho \lvert \tilde \ell_t'\rvert x_t)\le x_t,
\end{align*} 
where the last line uses the assumption that $x_t>0$.
\end{proof}

\subsection{\pone: $x_t$ Decreases Fast while $y_t$ Remains Large}
The definition of \pone is very similar to the one of SAM -- and the conclusion is also analogue, although the proofs are slightly different because of the different update rule.
\begin{definition}[\pone]\label{def:definition of Phase I USAM}
Let $\tone$ be the largest time  such that $x_t> \frac 12\sqrt{\gamma\eta}$ for all $t\le \tone$.
We denote by ``\textbf{\pone}'' the iterations $[0, \tone ]$.
\end{definition}

\begin{theorem}[Main Conclusion of \pone; USAM Case]\label{lem:phase1 USAM}
For $\tzero =\Theta(\min\{\eta^{-1}\rho^{-1},\eta^{-2}\})$,
we have $y_t
\geq \frac 12\sqrt{\gamma/\eta}$ for all $t\le \min\{\tzero,\tone\}$ under the conditions of \autoref{thm:USAM ell(xy) formal}.
Moreover, we have $x_{t+1}\le x_t-\frac 12\sqrt{\cy \gamma \eta}$ for all $t\le \min\{\tzero,\tone\}$, which consequently infers $\tone\le \tzero$ under the conditions of \autoref{thm:USAM ell(xy) formal}, i.e., $\min\{\tzero,\tone\}$ is just $\tone$. This shows the first claim of \autoref{thm:USAM ell(xy) formal}.
\end{theorem}
\begin{proof}
This theorem is a combination of \autoref{lem:y_t cannot be too small in I USAM} and \autoref{lem:x_t decreases fast in I USAM}.
\end{proof}

Similar to \autoref{lem:y_t cannot be too small in I SAM}, $y_t$ in USAM also cannot be too small in the first few iterations:
\begin{lemma}\label{lem:y_t cannot be too small in I USAM}
There exists $\tzero =\Theta(\min\{\eta^{-1}\rho^{-1},\eta^{-2}\})$ such that for $t\le  \tzero $,  we have  $|y_t|
\geq \frac 12\sqrt{\gamma/\eta}$. 
Assuming $\eta\le \frac 12$, then $y_t \geq \frac 12 \sqrt{\gamma / \eta}$ for $t\leq \min\{\tone,\tzero\}$.
\end{lemma}
\begin{proof}
The proof idea follows from \autoref{lem:y_t cannot be too small in I SAM}. Let $\tzero$ be defined as follows:
\begin{align}
\tzero \triangleq \max\left\{t~:~ \left (1-\eta^2-2\eta \rho \right )^t \geq \frac{1}{4}\right\}\,.
\end{align}
Then, it follows that $\tzero =\Theta(\min\{\eta^{-1}\rho^{-1},\eta^{-2}\})$.
We first show that $y_t^2$ cannot be too small until $\tzero$:
\begin{claim}\label{claim:square of y_t LB USAM}
For all $t\le \tzero$, $y_t^2\ge \frac 14 \gamma/\eta$.
\end{claim}
\begin{proof}
Prove by induction. Initially, $y_0^2=\cy \gamma /  \eta\ge \frac 14 \gamma / \eta$. Then consider some $t<\tzero$ such that $y_{t'}^2\ge \frac 14 \gamma / \eta$ for all $t'\le t$. By \autoref{eq:USAM simplified}, we have
\begin{align*}
y_{t+1}^2-x_{t+1}^2&=\left (y_t-\eta \ell_t'\cdot (x_t+\rho \tilde \ell_t' y_t)\right )^2-\left (x_t-\eta \ell_t'\cdot (y_t+\rho \tilde \ell_t' x_t)\right )^2\\
&=(y_t^2-x_t^2)+(\eta \ell_t')^2 \left ((x_t+\rho \tilde \ell_t' y_t)^2-(y_t+\rho \tilde \ell_t' x_t)^2\right )+\\
&\quad 2x_t \cdot \eta \ell_t'\cdot (y_t+\rho \tilde \ell_t' x_t)-2y_t \cdot \eta \ell_t'\cdot (x_t+\rho \tilde \ell_t' y_t)\\
&=(y_t^2-x_t^2)+(\eta \ell_t')^2(x_t^2-y_t^2)+(\eta \ell_t')^2(\rho \tilde \ell_t')^2 (y_t^2-x_t^2)+2\eta \ell_t'\cdot \rho \tilde \ell_t'(x_t^2-y_t^2)\\
&=\left (1-(\eta \ell_t')^2+(\eta \ell_t')^2(\rho \tilde \ell_t')^2-2\eta \ell_t'\cdot \rho \tilde \ell_t'\right )(y_t^2-x_t^2).
\end{align*}

Using Assumption \ref{A:cvx_lips}, we know that $\lvert \ell_t'\rvert\le 1$ and $\lvert \tilde \ell_t'\rvert\le 1$, giving
\begin{equation*}
y_{t+1}^2-x_{t+1}^2\ge \left (1-\eta^2-2\eta \rho\right )(y_t^2-x_t^2)\ge \cdots \ge \left (1-\eta^2-2\eta \rho\right )^{t+1}(y_0^2-x_0^2).
\end{equation*}

By definition of $\tzero$ and condition that $t<\tzero$, we have $y_{t+1}^2\ge y_{t+1}^2-x_{t+1}^2\ge \frac 14 \gamma / \eta$.
\end{proof}

\begin{claim}\label{claim:raw y_t LB USAM}
Suppose that $\sqrt{\cy \gamma \eta}\le \frac 12\sqrt{\gamma/\eta}$, i.e., $\eta\le \frac 12\sqrt \cy$. Then $y_t\ge \frac 12\sqrt{\gamma/\eta}$ for all $t\le t_0$.
\end{claim}
\begin{proof}
We still consider the maximum single-step difference in $y_t$. Suppose that for some $t'<\min\{\tzero,\tone\}$, $y_{t'}\ge \frac 12\sqrt{\gamma/\eta}$ for all $t'\le t$. According to \autoref{eq:USAM simplified} and Assumption \ref{A:cvx_lips},
\begin{equation*}
y_{t+1}=y_t-\eta \ell_t'\cdot (x_t+\rho \tilde \ell_t'y_t)\ge y_t-\eta (x_t+\rho y_t)\ge \left (1-\eta \rho\right )y_t-\eta \sqrt{\cy \gamma/\eta},
\end{equation*}
where the last inequality used \autoref{lem:monotone USAM} to conclude that $x_t\le x_0\le y_0=\sqrt{\cy \gamma/\eta}$. Hence, as the first term is non-negative and $\sqrt{\cy \gamma \eta}\le \frac 12 \sqrt{\gamma/\eta}$, we must have $y_{t+1}\ge 0$, which implies $y_{t+1}\ge \frac 12\sqrt{\gamma/\eta}$ according to \autoref{claim:square of y_t LB USAM}.
\end{proof}

Putting these two claims together gives our conclusion. Note that the condition $\eta\le \frac 12\sqrt \cy$ in \autoref{claim:raw y_t LB USAM} can be inferred from the assumptions $\eta\le \frac 12$ and $\cy \ge 1$ made in \autoref{thm:USAM ell(xy) formal}.
\end{proof}

After showing that $y_t$ never becomes too small, we are ready to show that $x_t$ decreases fast enough.
\begin{lemma}\label{lem:x_t decreases fast in I USAM}
For  $t\leq \min\{\tzero,\tone\}$,  we have 
\begin{align*}
x_{t+1} \le x_t-\frac 12\sqrt{\gamma\eta}\,.
\end{align*} 
In particular, if $\eta$ is sufficiently small s.t. $2\sqrt \cy \eta^{-1}\leq \tzero$,
then we must have $\tone \leq 2\sqrt \cy \eta^{-1}-1<\tzero$. 
\end{lemma}
\begin{proof}
Since $x_t>0$ for all $t\leq \tone$, we have $\sign(\ell'_t) = \sign(x_ty_t) = \sign(y_t)$ from \autoref{lem:sign USAM} \autoref{lem:y_t cannot be too small in I USAM}, so the USAM update \eqref{eq:USAM simplified} simplifies as follows according to Assumption \ref{A:cvx_lips} and \autoref{lem:y_t cannot be too small in I USAM}
\begin{equation*}
x_{t+1}=x_t-\eta \lvert \ell_t'\rvert \cdot (\lvert y_t\rvert+\rho \lvert \ell_t'\rvert \lvert x_t\rvert)\le x_t-\eta \lvert y_t\rvert\le x_t-\frac 12\sqrt{\gamma \eta}. 
\end{equation*}
Let $t_1'=\frac{\sqrt{\cy \gamma / \eta}}{\frac 12\sqrt{\gamma \eta}}=2\sqrt \cy \eta^{-1}$, then $t_1'-1<t_0$ and thus $x_{t_1'}<\frac 12\sqrt{\gamma \eta}$, i.e., $t_1<t_0$ must holds. 
\end{proof} 

\subsection{\pthree: $y_t$ Gets Trapped Above the Origin}
Now, we are going to consider the final-stage behavior of USAM.

\begin{theorem}[Main Conclusion of \pthree; USAM Case]\label{lem:phase3 usam}
After \pone, we always have $\lvert x_t\rvert\le \sqrt{\cy \gamma \eta}$ and $\lvert y_t\rvert\le \sqrt{\cy \gamma/\eta}$. Moreover, once we have $\eta (1+\rho y_\tenter^2)y_\tenter^2=2-\epsilon$ for some $\tenter\ge \tone$ and $\epsilon>0$, we must have $\lvert x_{t+1}\rvert\le \exp(-\Omega(\epsilon))\lvert x_t\rvert$ for all $t\ge \tenter$. This consequently infers that $y_\infty^2$, which is defined as $\liminf_{t\to \infty} y_t^2$, satisfies $y_\infty^2\ge (1-4 \cy \gamma (\eta + \rho \cy \gamma)^2 \epsilon^{-1}) y_\tenter^2$. As $\epsilon$ is a constant independent of $\eta$ and $\rho$ and can be arbitrarily close to $0$, this shows the second claim of \autoref{thm:USAM ell(xy) formal}.
\end{theorem}

\begin{proof}
The first conclusion is an analog of \autoref{lem:bounded x_t in Phase II} and \autoref{lem:bounded y_t in Phase II} (which allows a simpler analysis thanks to the removal of normalization), as we will show in \autoref{lem:bounded x_t and y_t in Phase II USAM}. The second part requires a similar (but much more sophisticated) analysis to Lemma 15 of \citet{ahn2022learning}, which we will cover in \autoref{lem:stopping position of USAM}.
\end{proof}

\begin{lemma}\label{lem:bounded x_t and y_t in Phase II USAM}
Suppose that $\lvert x_t\rvert\le \sqrt{\cy \gamma \eta}$, $\lvert y_t\rvert\le \sqrt{\cy \gamma / \eta}$. Assuming $\eta \rho \le 1$, then $\lvert x_{t+1}\rvert\le \sqrt{\cy \gamma \eta}$ and $\lvert y_{t+1}\rvert\le \sqrt{\cy \gamma / \eta}$ as well. Furthermore, $\lvert x_t\rvert\le \sqrt{\cy \gamma \eta}$ and $\lvert y_t\rvert\le \sqrt{\cy \gamma / \eta}$ hold for all $t>\tone$.
\end{lemma}
\begin{proof}
Suppose that $x_{t}\ge 0$ without loss of generality. The case where $x_{t+1}\ge 0$ is trivial by \autoref{lem:monotone USAM}. Otherwise, using \autoref{lem:sign USAM} and \autoref{eq:USAM simplified}, we can write
\begin{equation*}
x_{t+1}=x_t-\eta \ell_t'\cdot (y_t+\rho \tilde \ell_t' x_t)=x_t-\eta \lvert \ell_t'\rvert\cdot (\lvert y_t\rvert+\rho \lvert \tilde \ell_t'\rvert x_t).
\end{equation*}

By Assumption \ref{A:cvx_lips}, $\lvert \ell_t'\rvert$ and $\lvert \tilde \ell_t'\rvert$ are bounded by $1$. By the condition that $\eta \rho \le 1$,
\begin{equation*}
x_{t+1}\ge (1-\eta \rho)x_t-\eta \lvert y_t\rvert\ge -\sqrt{\cy \gamma \eta},
\end{equation*}
where we used $\lvert y_t\rvert\le sqrt{\cy\gamma/\eta}$. Similarly, for $y_{t+1}$, only considering the case where $y_t\ge 0$ and $y_{t+1}\le 0$ suffices. We have the following by symmetry
\begin{equation*}
y_{t+1}=y_t-\eta \ell_t'\cdot (x_t+\rho \tilde \ell_t' y_t)\ge y_t-\eta (\lvert x_t\rvert+\rho y_t)=(1-\eta \rho) y_t-\eta \lvert x_t\rvert\ge -\eta^{1.5} \sqrt{\cy \gamma}\ge -\sqrt{\cy \gamma /\eta},
\end{equation*}
where we used $\lvert x_t\rvert\le \sqrt{\cy \gamma \eta}$.

The second part of the conclusion is done by induction.
According to the definition of \pone, we have $\lvert x_{\tone+1}\rvert\le \frac 12\sqrt{\gamma\eta}$. As $\cy \ge 1\ge \frac 14$, we consequently have $\lvert x_t\rvert\le \sqrt{\cy \gamma \eta}$ for all $t>\tone$ from the first part of the conclusion.
Regarding $\lvert y_t\rvert$, recall \autoref{lem:y_t cannot be too small in I USAM} infers $y_t>0$ for all $t\le \tone+1$ and \autoref{lem:monotone USAM} infers the monotonicity of $y_t$, we have $y_{\tone+1}\le y_0=\sqrt{\cy \gamma /\eta}$. Hence, $\lvert y_t\rvert\le \sqrt{\cy \gamma / \eta}$ for all $t>\tone$ as well.
\end{proof}

Before showing the ultimate conclusion \autoref{lem:stopping position of USAM}, we first show the following single-step lemma:
\begin{lemma}
Suppose that $\eta (1+\rho y_t^2)y_t^2<2$ and $\lvert x_t\rvert\le \sqrt{\cy \gamma \eta}$ for some $t$. Define $\epsilon_t=2-\eta (1+\rho y_t^2)y_t^2$ (then we must have $\epsilon_t\in (0,2)$). Then we have
\begin{equation}\label{eq:decay of x_t in III USAM}
\lvert x_{t+1}\rvert\le \lvert x_t\rvert \exp\left (-\min\left \{(1+\rho \cy \gamma \eta) \epsilon_t,\frac{2-\epsilon_t}{8}\right \}\right ).
\end{equation}

\end{lemma}
\begin{proof}
Without loss of generality, assume that $x_t>0$. From \autoref{eq:USAM simplified}, we can write:
\begin{equation*}
x_{t+1}=x_t-\eta \ell_t'\cdot (y_t+\rho \tilde \ell_t' x_t)=x_t-\eta \lvert \ell_t'\rvert\cdot (\lvert y_t\rvert+\rho \lvert \tilde \ell_t'\rvert x_t),
\end{equation*}
where we used $\sgn(\ell_t')=\sgn(\tilde \ell_t')=\sgn(x_ty_t)=\sgn(y_t)$ (\autoref{lem:sign USAM}). By Assumption \ref{A:origin_rate},
\begin{align}
x_{t+1}&\ge x_t-\eta \left \lvert \big (x_t+\rho \tilde \ell_t' y_t\big )\big (y_t+\rho \tilde \ell_t' x_t\big )\right \rvert\big  (\lvert y_t\rvert+\rho \lvert x_ty_t\rvert x_t\big ) \nonumber\\
&=x_t-\eta \big (x_t+\rho \lvert \tilde \ell_t'\rvert \lvert y_t\rvert\big )\big (\lvert y_t\rvert+\rho \lvert \tilde \ell_t'\rvert x_t\big )\big (\lvert y_t\rvert+\rho x_t^2 \lvert y_t\rvert\big ) \nonumber\\
&\ge x_t-\eta (x_t+\rho x_t y_t^2)(\lvert y_t\rvert+\rho x_t^2 \lvert y_t\rvert)\left (\lvert y_t\rvert+\rho x_t^2 \lvert y_t\rvert\right ) \nonumber\\
&=\left (1-\eta (1+\rho y_t^2)(1+\rho x_t^2)^2y_t^2\right )x_t \nonumber\\
&\ge\left (1-(1+\rho \cy \gamma \eta)\eta (1+\rho y_t^2)y_t^2\right )x_t\nonumber\\
&\ge -\left (1-(1+\rho \cy \gamma \eta)\epsilon_t\right )x_t, \label{eq:dynamics of x_t III USAM}
\end{align}
where we used $\lvert x_t\rvert\le \sqrt{\cy \gamma \eta}$.
For the other direction, we have the following by Assumption \ref{A:bdd_der}:
\begin{align}
x_{t+1}&\le x_t-\eta \frac{\lvert (x_t+\rho \tilde \ell_t' y_t)(y_t+\rho \tilde \ell_t' x_t)\rvert}{2}\left (\lvert y_t\rvert+\rho \frac{\lvert x_ty_t\rvert}{2} x_t\right ) \nonumber\\
&=x_t-\eta \frac{(x_t+\rho \lvert \tilde \ell_t'\rvert \lvert y_t\rvert)(\lvert y_t\rvert+\rho \lvert \tilde \ell_t'\rvert x_t)}{2}\left (\lvert y_t\rvert+\rho \frac{x_t^2}{2} \lvert y_t\rvert\right ) \nonumber\\
&\le x_t-\eta \frac{(x_t+\rho x_t y_t^2)(\lvert y_t\rvert+\rho \lvert y_t\rvert x_t^2)}{8}\left (\lvert y_t\rvert+\rho \frac{x_t^2}{2} \lvert y_t\rvert\right ) \nonumber\\
&=\left (1-\frac \eta 8 (1+\rho y_t^2)(1+\rho x_t^2) \left (1+\rho \frac{x_t^2}{2}\right )y_t^2\right )x_t \nonumber\\
&\le \left (1-\frac \eta 8 (1+\rho y_t^2)y_t^2\right )x_t=\left (1-\frac{2-\epsilon}{8}\right )x_t,\label{eq:dynamics of y_t in II USAM}
\end{align}
where we used $(1+\rho x_t^2)\ge 1$ and $(1+\rho \frac{x_t^2}{2})\ge 1$.
\autoref{eq:decay of x_t in III USAM} follows from \autoref{eq:dynamics of x_t III USAM} and \autoref{eq:dynamics of y_t in II USAM}. 
\end{proof}

\begin{lemma}\label{lem:stopping position of USAM}
Let $\tenter$ be such that i) $\eta(1+\rho y_\tenter^2)y_\tenter^2=2-\epsilon$ where $\epsilon\in (0,\frac 29)$ is a constant, and ii) $\lvert x_\tenter\rvert\le \sqrt{\cy \gamma \eta}$. Then we have the following conclusion on $\liminf_{t\to \infty}y_t^2$, denoted by $y_\infty^2$ in short:
\begin{equation*}
y_\infty^2=\liminf_{t\to \infty} y_t^2\ge \big (1-4 \cy \gamma (\eta + \rho \cy \gamma)^2 \epsilon^{-1}\big )y_\tenter^2.
\end{equation*}
\end{lemma}
While the $\epsilon^{-1}$ looks enormous, it is a constant independent of $\eta$ and $\rho$; in other words, we are allowed to set $\epsilon$ as close to zero as we want. As we only consider the dependency on $\eta$ and $\rho$, we can abbreviate this conclusion as $y_\infty^2\ge (1-\mathcal O(\eta^2+\rho^2))y_\tenter^2$, as we claim in the main text.
\begin{proof} 
In analog to \autoref{eq:dynamics of x_t III USAM}, we derive the following for $y_t$:
\begin{equation*}
y_{t+1}^2\ge \bigg (1-\eta\left (1+\rho y_t^2\right )^2 (1+\rho x_t^2) x_t^2\bigg )^2 y_t^2\ge \bigg (1-2\eta (1+\rho \cy \gamma / \eta)^2 (1+\rho \cy \gamma \eta) x_t^2\bigg )y_t^2,
\end{equation*}
where the second inequality uses \autoref{lem:bounded x_t and y_t in Phase II USAM}. Let $d_t=y_\tenter^2-y_t^2$, then we have
\begin{equation*}
d_{t+1}\le d_t+2\eta (1+\rho \cy \gamma / \eta)^2 (1+\rho \cy \gamma \eta) x_t^2 y_t^2\le d_t+4\eta^{-1} (\eta + \rho \cy \gamma)^2 y_\tenter^2 x_t^2,
\end{equation*}
where we used the assumption that $\rho \cy \gamma \eta \le 1$ and the fact that $y_t^2$ is monotonic (thus $y_t^2\le y_\tenter^2$).

According to \autoref{eq:decay of x_t in III USAM}, we have $\lvert x_{t+1}\rvert\le \lvert x_t\rvert \exp(-\Omega(\epsilon))$ for all $t\ge \ttwo$. Hence, we have
\begin{equation*}
d_\infty =\limsup_{t\to \infty} d_t \le d_{\tenter} + \sum_{t=\tenter}^\infty 4\eta^{-1} (\eta + \rho \cy \gamma)^2 y_\tenter^2 x_t^2=4\eta^{-1} (\eta + \rho \cy \gamma)^2 y_\tenter^2 \cdot \epsilon^{-1} x_{\tenter}^2,
\end{equation*}
where we used $d_{\tenter}=0$ (by definition) and the sum of geometric series. Plugging back $x_{\tenter}^2\le \cy \gamma \eta$,
\begin{equation*}
y_\infty^2=\liminf_{t\to \infty}y_t^2\ge y_{\tenter}^2 - 4 \cy \gamma (\eta + \rho \cy \gamma)^2 y_\tenter^2,
\end{equation*}
as claimed. 
\end{proof}

\section{Omitted Proof of USAM Over General PL functions}
\label{sec:appendix_general}

\begin{theorem}[Formal Version of \autoref{thm:USAM general}]\label{thm:USAM general formal}
For any $\mu$-PL and $\beta$-smooth loss function $\loss$, for any learning rate $\eta <\nicefrac 1\beta$ and $\rho< \nicefrac{1}{\beta}$, for any initialization $w_0$, the following holds for USAM:
\begin{equation*}
\lVert w_t-w_0\rVert\le \eta (1+\beta\rho)  \sqrt{\frac{2\beta^2 }{\mu}}\left(1-2\mu\eta (1  - \rho  \beta)\left (\frac{\eta \beta}{2}(1-\rho\beta)\right ) \right)^{-\nicefrac 12} \sqrt{\loss(w_0)-\loss^\ast},\quad \forall t\ge 0,
\end{equation*}
where $\loss^\ast$ is the short-hand notation for $\min_w \loss(w)$.
\end{theorem}

We first state the following useful result by \citet[Theorem 10]{andriushchenko2022towards}.  
\begin{lemma}[Descent Lemma of USAM over Smooth and PL Losses]\label{lem:USAM PL and Smooth}
For any $\beta$-smooth and $\mu$-PL loss function $\loss$, for any learning rate $\eta <\nicefrac 1\beta$ and $\rho< \nicefrac{1}{\beta}$, the following holds for USAM:
\begin{align*}
\loss(w_{t}) - \loss^\star \leq   \left(1-2\mu\eta (1  - \rho  \beta)\left ( 1- \frac{\eta \beta}{2}(1-\rho\beta)\right ) \right)^t 
(\loss(w_{0}) - \loss^\star ),\quad \forall t\ge 0.
\end{align*}
\end{lemma}

\begin{proof}[Proof of \autoref{thm:USAM general formal}]
We follow the convention in \cite{karimi2016linear}: Let $\mathcal{X}^\star$ be the set of global minima and $x_p$ be the projection of $x$ onto
the solution set $\mathcal{X}^\star$.
From $\beta$-smoothness, it follows that
\begin{align*}
\norm{\nabla \loss(x)} = 
\norm{\nabla \loss(x)- \nabla \loss(x_p)} \leq  \beta \norm{x- x_p},\quad \forall x.
\end{align*}
Now since $\loss$ is $\beta$-smooth and $\mu$-PL, Theorem 2 from \citep{karimi2016linear} implies that the quadratic growth condition holds, i.e., 
\begin{align*}
\frac{2}{\mu} \left( \loss(x) -\loss^\star\right) \geq \norm{x-x_p}^2,\quad \forall x.
\end{align*}
Thus, it follows that 
\begin{align*}
\norm{\nabla \loss(x)}^2 \leq \frac{2\beta^2 }{\mu} \left( \loss(x) -\loss^\star\right),\quad \forall x.
\end{align*}
Moreover, from $\beta$-smoothness, we have 
\begin{align*}
\norm{\nabla \loss(x+\rho \nabla \loss(x))} \leq \norm{\nabla\loss(x)} + \beta \norm{\rho \nabla\loss(x)} = (1+\beta\rho)\norm{\nabla \loss(x)},\quad \forall x.
\end{align*}
Thus, by the update rule of USAM \eqref{eq:definition of USAM}, it follows that 
\begin{align*}
\norm{w_t-w_0}&\leq \eta \sum_{i=0}^{t-1} \norm{\nabla \loss(w_i+\rho \nabla \loss(w_i))} \\
&\leq  \eta (1+\beta\rho) \sum_{i=0}^{t-1} \norm{\nabla \loss(w_i)}\\
& \leq  \eta (1+\beta\rho) \sum_{i=0}^{t-1} \sqrt{\frac{2\beta^2 }{\mu} \left( \loss(w_i) -\loss^\star\right)}\\
&= \eta (1+\beta\rho)  \sqrt{\frac{2\beta^2 }{\mu}} \sum_{i=0}^{t-1} \sqrt{\loss(w_i) -\loss^\star}. 
\end{align*}
Now we only to invoke the USAM descent lemma stated before, i.e., \autoref{lem:USAM PL and Smooth}, giving
\begin{align*}
\sum_{i=0}^{t-1} \sqrt{\loss(w_i)-\loss^\ast}&\le \sum_{i=0}^{t-1} \left(1-2\mu\eta (1  - \rho  \beta)\left ( 1- \frac{\eta \beta}{2}(1-\rho\beta)\right ) \right)^{\nicefrac i2} \sqrt{\loss(w_0)-\loss^\ast}\\
&\le \left(1-2\mu\eta (1  - \rho  \beta)\left (\frac{\eta \beta}{2}(1-\rho\beta)\right ) \right)^{-\nicefrac 12} \sqrt{\loss(w_0)-\loss^\ast}.
\end{align*}
Putting the last two inequalities together then give our conclusion.
\end{proof} 

\end{document}

%% file: command.tex
\newcommand{\R}{\mathbb{R}} 

\newcommand{\eps}{\varepsilon}

\newcommand{\bi}{\begin{enumerate}}
\newcommand{\ei}{\end{enumerate}}

\newcommand{\bli}{\begin{list}{{\tiny $\blacksquare$}}{\leftmargin=1.5em}
\setlength{\itemsep}{-1pt}
}
\newcommand{\blii}{\begin{list}{{  $\bullet$}}{\leftmargin=0.5em}
\setlength{\itemsep}{-1pt}
}
\newcommand{\eli}{\end{list}}

\DeclareMathOperator{\sgn}{sign}

\newcommand{\inp}[2]{\left\langle #1,#2 \right\rangle}

\newcommand{\tzero}{t_0}  
\newcommand{\tone}{t_1}  
\newcommand{\ttwo}{t_2}  
\newcommand{\tenter}{\mathfrak t }

\newcommand{\cy}{C}

\newcommand{\pone}{Initial Phase\xspace}
\newcommand{\ptwo}{Middle Phase\xspace}
\newcommand{\pthree}{Final Phase\xspace}
\newcommand{\loss}{{\mathcal L}}

\definecolor{eblue}{RGB}{100, 174, 118}

\definecolor{egreen}{RGB}{100, 118, 174}
\definecolor{persimmon}{rgb}{0.93, 0.35, 0.0}
\definecolor{darkpastelpurple}{rgb}{0.59, 0.44, 0.84}

\newcommand{\norm}[1]{\left\| #1 \right\|}

%% file: _macros.tex
\usepackage{comment,url,algorithm,algorithmic,graphicx, relsize, subcaption}
\usepackage{amssymb,amsfonts,amsmath,amsthm,amscd,dsfont,mathrsfs,mathtools,microtype,nicefrac,pifont}
\usepackage{float, psfrag,epsfig,color,url,hyperref}
\usepackage{upgreek} 
\usepackage{epstopdf,bbm,mathtools,enumitem}
\usepackage[toc,page]{appendix}
\usepackage[mathscr]{euscript} 
\usepackage{xspace} 

\def\balign#1\ealign{\begin{align}#1\end{align}}
\def\baligns#1\ealigns{\begin{align*}#1\end{align*}}
\def\balignat#1\ealign{\begin{alignat}#1\end{alignat}}
\def\balignats#1\ealigns{\begin{alignat*}#1\end{alignat*}}
\def\bitemize#1\eitemize{\begin{itemize}#1\end{itemize}}
\def\benumerate#1\eenumerate{\begin{enumerate}#1\end{enumerate}}

\newenvironment{talign*}
 {\csname align*\endcsname}
 {\endalign}
\newenvironment{talign}
 {\csname align\endcsname}
 {\endalign}

\def\balignst#1\ealignst{\begin{talign*}#1\end{talign*}}
\def\balignt#1\ealignt{\begin{talign}#1\end{talign}}
\let\originalleft\left
\let\originalright\right
\renewcommand{\left}{\mathopen{}\mathclose\bgroup\originalleft}
\renewcommand{\right}{\aftergroup\egroup\originalright}

\def\tinycitep*#1{{\tiny\citep*{#1}}}
\def\tinycitealt*#1{{\tiny\citealt*{#1}}}
\def\tinycite*#1{{\tiny\cite*{#1}}}
\def\smallcitep*#1{{\scriptsize\citep*{#1}}}
\def\smallcitealt*#1{{\scriptsize\citealt*{#1}}}
\def\smallcite*#1{{\scriptsize\cite*{#1}}}

\def\<{\left\langle} % Angle brackets
\def\>{\right\rangle}

\DeclareSymbolFont{rsfs}{U}{rsfs}{m}{n}
\DeclareSymbolFontAlphabet{\mathscrsfs}{rsfs}
\providecommand{\diag}{\mathop\mathrm{diag}}

\providecommand{\sign}{\mathop\mathrm{sign}}
\usepackage{cleveref}
\Crefformat{equation}{Eq. #2(#1)#3}
\Crefrangeformat{equation}{Eqs. #3(#1)#4 to #5(#2)#6}
\Crefmultiformat{equation}{Eqs. #2(#1)#3}{ and #2(#1)#3}{, #2(#1)#3}{ and #2(#1)#3}
\Crefrangemultiformat{equation}{Eqs. #3(#1)#4 to #5(#2)#6}{ and #3(#1)#4 to #5(#2)#6}{, #3(#1)#4 to #5(#2)#6}{ and #3(#1)#4 to #5(#2)#6}

\ifdefined\nonewproofenvironments\else
\ifdefined\ispres\else
\newtheorem{theorem}{Theorem}
\newtheorem{lemma}[theorem]{Lemma}
\newtheorem{corollary}[theorem]{Corollary}
\newtheorem{definition}[theorem]{Definition}

\newenvironment{proof-sketch}{\noindent\textbf{Proof Sketch}
  \hspace*{1em}}{\qed\bigskip\\}
\newenvironment{proof-idea}{\noindent\textbf{Proof Idea}
  \hspace*{1em}}{\qed\bigskip\\}
\newenvironment{proof-of-lemma}[1][{}]{\noindent\textbf{Proof of Lemma {#1}}
  \hspace*{1em}}{\qed\\}
\newenvironment{proof-of-theorem}[1][{}]{\noindent\textbf{Proof of Theorem {#1}}
  \hspace*{1em}}{\qed\\}
\newenvironment{proof-attempt}{\noindent\textbf{Proof Attempt}
  \hspace*{1em}}{\qed\bigskip\\}

\fi

\theoremstyle{definition}

\newtheorem*{remark*}{Remark}
\newtheorem{claim}[theorem]{Claim}
\Crefname{claim}{Claim}{Claims}
\fi
\hypersetup{
  colorlinks,
  linkcolor={red!50!black},
  citecolor={blue!50!black},
  urlcolor={blue!80!black}
}

%% file: main_neurips.bbl
\begin{thebibliography}{45}
\providecommand{\natexlab}[1]{#1}
\providecommand{\url}[1]{\texttt{#1}}
\expandafter\ifx\csname urlstyle\endcsname\relax
  \providecommand{\doi}[1]{doi: #1}\else
  \providecommand{\doi}{doi: \begingroup \urlstyle{rm}\Url}\fi

\bibitem[Abernethy et~al.(2023)Abernethy, Agarwal, Marinov, and
  Warmuth]{abernethy2023mechanism}
Jacob Abernethy, Alekh Agarwal, Teodor~V Marinov, and Manfred~K Warmuth.
\newblock A mechanism for sample-efficient in-context learning for sparse
  retrieval tasks.
\newblock \emph{arXiv preprint arXiv:2305.17040}, 2023.

\bibitem[Agarwala and Dauphin(2023)]{agarwala2023sam}
Atish Agarwala and Yann~N Dauphin.
\newblock Sam operates far from home: eigenvalue regularization as a dynamical
  phenomenon.
\newblock \emph{arXiv preprint arXiv:2302.08692}, 2023.

\bibitem[Ahn et~al.(2022)Ahn, Zhang, and Sra]{ahn2022understanding}
Kwangjun Ahn, Jingzhao Zhang, and Suvrit Sra.
\newblock Understanding the unstable convergence of gradient descent.
\newblock In \emph{Proceedings of the 39th International Conference on Machine
  Learning}, volume 162 of \emph{Proceedings of Machine Learning Research},
  pages 247--257. PMLR, 2022.

\bibitem[Ahn et~al.(2023{\natexlab{a}})Ahn, Bubeck, Chewi, Lee, Suarez, and
  Zhang]{ahn2022learning}
Kwangjun Ahn, S{\'e}bastien Bubeck, Sinho Chewi, Yin~Tat Lee, Felipe Suarez,
  and Yi~Zhang.
\newblock Learning threshold neurons via the ``edge of stability''.
\newblock \emph{NeurIPS 2023 (arXiv:2212.07469)}, 2023{\natexlab{a}}.

\bibitem[Ahn et~al.(2023{\natexlab{b}})Ahn, Cheng, Daneshmand, and
  Sra]{ahn2023transformers}
Kwangjun Ahn, Xiang Cheng, Hadi Daneshmand, and Suvrit Sra.
\newblock Transformers learn to implement preconditioned gradient descent for
  in-context learning.
\newblock \emph{NeurIPS 2023 (arXiv:2306.00297)}, 2023{\natexlab{b}}.

\bibitem[Ahn et~al.(2023{\natexlab{c}})Ahn, Cheng, Song, Yun, Jadbabaie, and
  Sra]{ahn2023linear}
Kwangjun Ahn, Xiang Cheng, Minhak Song, Chulhee Yun, Ali Jadbabaie, and Suvrit
  Sra.
\newblock Linear attention is (maybe) all you need (to understand transformer
  optimization).
\newblock \emph{arXiv 2310.01082}, 2023{\natexlab{c}}.

\bibitem[Ahn et~al.(2023{\natexlab{d}})Ahn, Jadbabaie, and Sra]{ahn2023escape}
Kwangjun Ahn, Ali Jadbabaie, and Suvrit Sra.
\newblock How to escape sharp minima.
\newblock \emph{arXiv preprint arXiv:2305.15659}, 2023{\natexlab{d}}.

\bibitem[Allen-Zhu and Li(2023)]{allen2023physics}
Zeyuan Allen-Zhu and Yuanzhi Li.
\newblock Physics of language models: Part 1, context-free grammar.
\newblock \emph{arXiv preprint arXiv:2305.13673}, 2023.

\bibitem[Andriushchenko and Flammarion(2022)]{andriushchenko2022towards}
Maksym Andriushchenko and Nicolas Flammarion.
\newblock Towards understanding sharpness-aware minimization.
\newblock In \emph{International Conference on Machine Learning}, pages
  639--668. PMLR, 2022.

\bibitem[Arora et~al.(2022)Arora, Li, and Panigrahi]{arora2022understanding}
Sanjeev Arora, Zhiyuan Li, and Abhishek Panigrahi.
\newblock Understanding gradient descent on the edge of stability in deep
  learning.
\newblock In \emph{International Conference on Machine Learning}, pages
  948--1024. PMLR, 2022.

\bibitem[Bahri et~al.(2022)Bahri, Mobahi, and Tay]{bahri2022sharpness}
Dara Bahri, Hossein Mobahi, and Yi~Tay.
\newblock Sharpness-aware minimization improves language model generalization.
\newblock In \emph{Proceedings of the 60th Annual Meeting of the Association
  for Computational Linguistics (Volume 1: Long Papers)}, pages 7360--7371,
  2022.

\bibitem[Bartlett et~al.(2022)Bartlett, Long, and
  Bousquet]{bartlett2022dynamics}
Peter~L Bartlett, Philip~M Long, and Olivier Bousquet.
\newblock The dynamics of sharpness-aware minimization: Bouncing across ravines
  and drifting towards wide minima.
\newblock \emph{arXiv preprint arXiv:2210.01513}, 2022.

\bibitem[Behdin and Mazumder(2023)]{behdin2023sharpness}
Kayhan Behdin and Rahul Mazumder.
\newblock Sharpness-aware minimization: An implicit regularization perspective.
\newblock \emph{arXiv preprint arXiv:2302.11836}, 2023.

\bibitem[Blanc et~al.(2020)Blanc, Gupta, Valiant, and
  Valiant]{blanc2020implicit}
Guy Blanc, Neha Gupta, Gregory Valiant, and Paul Valiant.
\newblock Implicit regularization for deep neural networks driven by an
  ornstein-uhlenbeck like process.
\newblock In \emph{Conference on learning theory}, pages 483--513. PMLR, 2020.

\bibitem[Cohen et~al.(2021)Cohen, Kaur, Li, Kolter, and Talwalkar]{Cohen2021}
Jeremy Cohen, Simran Kaur, Yuanzhi Li, J~Zico Kolter, and Ameet Talwalkar.
\newblock Gradient descent on neural networks typically occurs at the edge of
  stability.
\newblock In \emph{International Conference on Learning Representations}, 2021.

\bibitem[Compagnoni et~al.(2023)Compagnoni, Orvieto, Biggio, Kersting, Proske,
  and Lucchi]{compagnoni2023sde}
Enea~Monzio Compagnoni, Antonio Orvieto, Luca Biggio, Hans Kersting,
  Frank~Norbert Proske, and Aurelien Lucchi.
\newblock An sde for modeling sam: Theory and insights.
\newblock \emph{arXiv preprint arXiv:2301.08203}, 2023.

\bibitem[Damian et~al.(2021)Damian, Ma, and Lee]{damian2021label}
Alex Damian, Tengyu Ma, and Jason~D. Lee.
\newblock Label noise {SGD} provably prefers flat global minimizers.
\newblock In \emph{Advances in Neural Information Processing Systems}, 2021.

\bibitem[Damian et~al.(2023)Damian, Nichani, and Lee]{damian2022self}
Alex Damian, Eshaan Nichani, and Jason~D Lee.
\newblock Self-stabilization: The implicit bias of gradient descent at the edge
  of stability.
\newblock In \emph{International Conference on Learning Representations}, 2023.

\bibitem[Du et~al.(2022)Du, Yan, Feng, Zhou, Zhen, Goh, and
  Tan]{du2022efficient}
Jiawei Du, Hanshu Yan, Jiashi Feng, Joey~Tianyi Zhou, Liangli Zhen, Rick
  Siow~Mong Goh, and Vincent Tan.
\newblock Efficient sharpness-aware minimization for improved training of
  neural networks.
\newblock In \emph{International Conference on Learning Representations}, 2022.

\bibitem[Dziugaite and Roy(2017)]{dziugaite2017computing}
Gintare~Karolina Dziugaite and Daniel~M Roy.
\newblock Computing nonvacuous generalization bounds for deep (stochastic)
  neural networks with many more parameters than training data.
\newblock In \emph{Proceedings of the Thirty-Third Conference on Uncertainty in
  Artificial Intelligence}, 2017.

\bibitem[Foret et~al.(2021)Foret, Kleiner, Mobahi, and
  Neyshabur]{foret2020sharpness}
Pierre Foret, Ariel Kleiner, Hossein Mobahi, and Behnam Neyshabur.
\newblock Sharpness-aware minimization for efficiently improving
  generalization.
\newblock In \emph{International Conference on Learning Representations}, 2021.

\bibitem[Garg et~al.(2022)Garg, Tsipras, Liang, and Valiant]{garg2022can}
Shivam Garg, Dimitris Tsipras, Percy~S Liang, and Gregory Valiant.
\newblock What can transformers learn in-context? a case study of simple
  function classes.
\newblock \emph{Advances in Neural Information Processing Systems},
  35:\penalty0 30583--30598, 2022.

\bibitem[Jastrzkebski et~al.(2017)Jastrzkebski, Kenton, Arpit, Ballas, Fischer,
  Bengio, and Storkey]{jastrzkebski2017three}
Stanislaw Jastrzkebski, Zachary Kenton, Devansh Arpit, Nicolas Ballas, Asja
  Fischer, Yoshua Bengio, and Amos Storkey.
\newblock Three factors influencing minima in {SGD}.
\newblock \emph{arXiv preprint arXiv:1711.04623}, 2017.

\bibitem[Jiang et~al.(2020)Jiang, Neyshabur, Mobahi, Krishnan, and
  Bengio]{jiang2019fantastic}
Yiding Jiang, Behnam Neyshabur, Hossein Mobahi, Dilip Krishnan, and Samy
  Bengio.
\newblock Fantastic generalization measures and where to find them.
\newblock In \emph{8th International Conference on Learning Representations},
  2020.

\bibitem[Karimi et~al.(2016)Karimi, Nutini, and Schmidt]{karimi2016linear}
Hamed Karimi, Julie Nutini, and Mark Schmidt.
\newblock Linear convergence of gradient and proximal-gradient methods under
  the polyak-{\l}ojasiewicz condition.
\newblock In \emph{Machine Learning and Knowledge Discovery in Databases:
  European Conference}, pages 795--811. Springer, 2016.

\bibitem[Kaur et~al.(2022)Kaur, Cohen, and Lipton]{kaur2022maximum}
Simran Kaur, Jeremy Cohen, and Zachary~C Lipton.
\newblock On the maximum hessian eigenvalue and generalization.
\newblock \emph{arXiv preprint arXiv:2206.10654}, 2022.

\bibitem[Keskar et~al.(2017)Keskar, Nocedal, Tang, Mudigere, and
  Smelyanskiy]{keskar2017large}
Nitish~Shirish Keskar, Jorge Nocedal, Ping Tak~Peter Tang, Dheevatsa Mudigere,
  and Mikhail Smelyanskiy.
\newblock On large-batch training for deep learning: Generalization gap and
  sharp minima.
\newblock In \emph{International Conference on Learning Representations}, 2017.

\bibitem[Kim et~al.(2023)Kim, Park, Choi, and Lee]{kim2023stability}
Hoki Kim, Jinseong Park, Yujin Choi, and Jaewook Lee.
\newblock Stability analysis of sharpness-aware minimization.
\newblock \emph{arXiv preprint arXiv:2301.06308}, 2023.

\bibitem[Kwon et~al.(2021)Kwon, Kim, Park, and Choi]{kwon2021asam}
Jungmin Kwon, Jeongseop Kim, Hyunseo Park, and In~Kwon Choi.
\newblock Asam: Adaptive sharpness-aware minimization for scale-invariant
  learning of deep neural networks.
\newblock In \emph{International Conference on Machine Learning}, pages
  5905--5914. PMLR, 2021.

\bibitem[Li et~al.(2018)Li, Ma, and Zhang]{li2018algorithmic}
Yuanzhi Li, Tengyu Ma, and Hongyang Zhang.
\newblock Algorithmic regularization in over-parameterized matrix sensing and
  neural networks with quadratic activations.
\newblock In \emph{Conference On Learning Theory}, pages 2--47. PMLR, 2018.

\bibitem[Li et~al.(2023)Li, Li, and Risteski]{li2023transformers}
Yuchen Li, Yuanzhi Li, and Andrej Risteski.
\newblock How do transformers learn topic structure: Towards a mechanistic
  understanding.
\newblock \emph{International Conference on Machine Learning (ICML)
  (arXiv:2303.04245)}, 2023.

\bibitem[Liu et~al.(2023)Liu, Ash, Goel, Krishnamurthy, and
  Zhang]{liu2022transformers}
Bingbin Liu, Jordan~T Ash, Surbhi Goel, Akshay Krishnamurthy, and Cyril Zhang.
\newblock Transformers learn shortcuts to automata.
\newblock \emph{ICLR (arXiv:2210.10749)}, 2023.

\bibitem[Liu et~al.(2020)Liu, Papailiopoulos, and Achlioptas]{liu2020bad}
Shengchao Liu, Dimitris Papailiopoulos, and Dimitris Achlioptas.
\newblock Bad global minima exist and sgd can reach them.
\newblock \emph{Advances in Neural Information Processing Systems},
  33:\penalty0 8543--8552, 2020.

\bibitem[Liu et~al.(2022)Liu, Mai, Chen, Hsieh, and You]{liu2022towards}
Yong Liu, Siqi Mai, Xiangning Chen, Cho-Jui Hsieh, and Yang You.
\newblock Towards efficient and scalable sharpness-aware minimization.
\newblock In \emph{Proceedings of the IEEE/CVF Conference on Computer Vision
  and Pattern Recognition}, pages 12360--12370, 2022.

\bibitem[Mi et~al.(2022)Mi, Shen, Ren, Zhou, Sun, Ji, and Tao]{mi2022make}
Peng Mi, Li~Shen, Tianhe Ren, Yiyi Zhou, Xiaoshuai Sun, Rongrong Ji, and
  Dacheng Tao.
\newblock Make sharpness-aware minimization stronger: A sparsified perturbation
  approach.
\newblock In \emph{Advances in Neural Information Processing Systems}, 2022.

\bibitem[Neyshabur et~al.(2017)Neyshabur, Bhojanapalli, McAllester, and
  Srebro]{neyshabur2017exploring}
Behnam Neyshabur, Srinadh Bhojanapalli, David McAllester, and Nati Srebro.
\newblock Exploring generalization in deep learning.
\newblock \emph{Advances in neural information processing systems}, 30, 2017.

\bibitem[Si and Yun(2023)]{si2023practical}
Dongkuk Si and Chulhee Yun.
\newblock Practical sharpness-aware minimization cannot converge all the way to
  optima.
\newblock \emph{NeurIPS 2023 (arXiv:2306.09850)}, 2023.

\bibitem[von Oswald et~al.(2023)von Oswald, Niklasson, Randazzo, Sacramento,
  Mordvintsev, Zhmoginov, and Vladymyrov]{von2022transformers}
Johannes von Oswald, Eyvind Niklasson, Ettore Randazzo, Jo{\~a}o Sacramento,
  Alexander Mordvintsev, Andrey Zhmoginov, and Max Vladymyrov.
\newblock Transformers learn in-context by gradient descent.
\newblock In \emph{International Conference on Machine Learning}, pages
  35151--35174. PMLR, 2023.

\bibitem[Wang et~al.(2022)Wang, Li, and Li]{wang2022analyzing}
Zixuan Wang, Zhouzi Li, and Jian Li.
\newblock Analyzing sharpness along gd trajectory: Progressive sharpening and
  edge of stability.
\newblock \emph{Advances in Neural Information Processing Systems},
  35:\penalty0 9983--9994, 2022.

\bibitem[Wen et~al.(2023)Wen, Ma, and Li]{wen2022does}
Kaiyue Wen, Tengyu Ma, and Zhiyuan Li.
\newblock How does sharpness-aware minimization minimize sharpness?
\newblock In \emph{International Conference on Learning Representations}, 2023.

\bibitem[Wu et~al.(2020)Wu, Xia, and Wang]{wu2020adversarial}
Dongxian Wu, Shu-Tao Xia, and Yisen Wang.
\newblock Adversarial weight perturbation helps robust generalization.
\newblock \emph{Advances in Neural Information Processing Systems},
  33:\penalty0 2958--2969, 2020.

\bibitem[Zhang et~al.(2022)Zhang, Backurs, Bubeck, Eldan, Gunasekar, and
  Wagner]{zhang2022unveiling}
Yi~Zhang, Arturs Backurs, S{\'e}bastien Bubeck, Ronen Eldan, Suriya Gunasekar,
  and Tal Wagner.
\newblock Unveiling transformers with lego: a synthetic reasoning task.
\newblock \emph{arXiv preprint arXiv:2206.04301}, 2022.

\bibitem[Zheng et~al.(2021)Zheng, Zhang, and Mao]{zheng2021regularizing}
Yaowei Zheng, Richong Zhang, and Yongyi Mao.
\newblock Regularizing neural networks via adversarial model perturbation.
\newblock In \emph{Proceedings of the IEEE/CVF Conference on Computer Vision
  and Pattern Recognition}, pages 8156--8165, 2021.

\bibitem[Zhong et~al.(2022)Zhong, Ding, Shen, Mi, Liu, Du, and
  Tao]{zhong2022improving}
Qihuang Zhong, Liang Ding, Li~Shen, Peng Mi, Juhua Liu, Bo~Du, and Dacheng Tao.
\newblock Improving sharpness-aware minimization with fisher mask for better
  generalization on language models.
\newblock In \emph{Findings of the Association for Computational Linguistics:
  EMNLP 2022}, pages 4064--4085, 2022.

\bibitem[Zhuang et~al.(2022)Zhuang, Gong, Yuan, Cui, Adam, Dvornek, s~Duncan,
  Liu, et~al.]{zhuang2022surrogate}
Juntang Zhuang, Boqing Gong, Liangzhe Yuan, Yin Cui, Hartwig Adam, Nicha~C
  Dvornek, James s~Duncan, Ting Liu, et~al.
\newblock Surrogate gap minimization improves sharpness-aware training.
\newblock In \emph{International Conference on Learning Representations}, 2022.

\end{thebibliography}
